\documentclass{article}

\usepackage{microtype}
\usepackage{graphicx}
\usepackage{subfigure}
\usepackage{booktabs} 

\usepackage{hyperref}

\usepackage{listings}
\usepackage[accepted]{icml2024}

\usepackage{amsmath}

\usepackage{amssymb}
\usepackage{mathtools}
\usepackage{amsthm}

\usepackage[capitalize,noabbrev]{cleveref}

\theoremstyle{plain}
\newtheorem{theorem}{Theorem}[section]

\newtheorem{corollary}[theorem]{Corollary}
\theoremstyle{definition}

\theoremstyle{remark}

\usepackage[textsize=tiny]{todonotes}

\icmltitlerunning{Algorithm of Thoughts}

\begin{document}

\twocolumn[
\icmltitle{Algorithm of Thoughts:\\ Enhancing Exploration of Ideas in Large Language Models}



\icmlsetsymbol{equal}{*}

\begin{icmlauthorlist}
\icmlauthor{Bilgehan Sel}{sch}
\icmlauthor{Ahmad Al-Tawaha}{sch}
\icmlauthor{Vanshaj Khattar}{sch}
\icmlauthor{Ruoxi Jia}{sch}
\icmlauthor{Ming Jin}{sch}
\end{icmlauthorlist}

\icmlaffiliation{sch}{Department of Electrical and Computer Engineering, Virginia Tech, Blacksburg, USA}
\icmlcorrespondingauthor{Bilgehan Sel}{bsel@vt.edu}

\icmlkeywords{Decision-making, Large Language Models, Language Models, Machine Learning, Algorithmic Thinking}

\vskip 0.3in
]



\printAffiliationsAndNotice{}  

\begin{abstract}
Current literature, aiming to surpass the ``Chain-of-Thought'' approach, often resorts to external modi operandi involving halting, modifying, and then resuming the generation process to boost Large Language Models' (LLMs) reasoning capacities. Due to their \emph{myopic perspective}, they escalate the number of query requests, leading to increased costs, memory, and computational overheads. Addressing this, we propose the \textit{Algorithm of Thoughts}---a novel strategy that propels LLMs through algorithmic reasoning pathways. By employing algorithmic examples fully in-context, this overarching view of the whole process exploits the innate recurrence dynamics of LLMs, expanding their idea exploration with merely one or a few queries. Our technique outperforms earlier single-query methods and even more recent multi-query strategies that employ an extensive tree search algorithms while using significantly fewer tokens. Intriguingly, our results suggest that instructing an LLM using an algorithm can lead to performance surpassing that of the algorithm itself, hinting at LLM's inherent ability to weave its intuition into optimized searches. We probe into the underpinnings of our method's efficacy and its nuances in application. The code and related content can be found in: \href{https://algorithm-of-thoughts.github.io/}{algorithm-of-thoughts.github.io}.
\end{abstract}

\section{Introduction}

\begin{figure*}[t]
\centering
\includegraphics[width=0.8\textwidth]{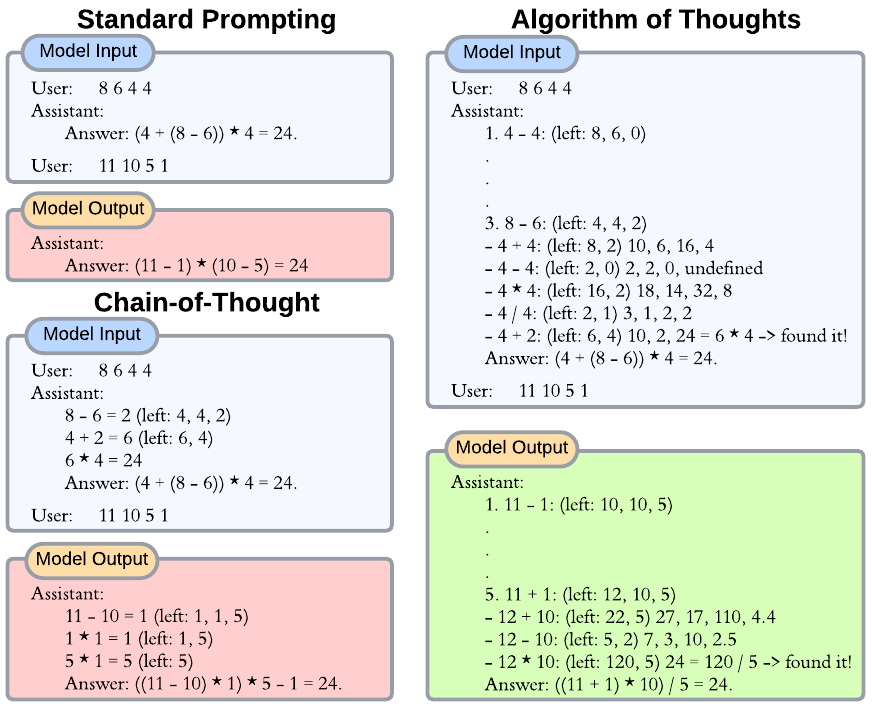}
\caption{Comparison between standard prompting, CoT, and AoT in the game of 24. While standard prompting aims for a direct answer, CoT sketches out the successive steps to the final solution. AoT's in-context example, distinct from CoT, integrates the search process, highlighted by markers `1',..., `3' as ``first operations'' guiding subtree exploration for the problem set `8 6 4 4'. For clarity, only a single in-context example is displayed, with a focus on the third subtree exploration. AoT produces prospective search steps (e.g., the subtree exploration `5. $11+1$') and evaluates potential subsequent steps to either progress towards a solution or retrace to another viable subtree.}
\label{fig: cot_vs_aot}
\end{figure*}

Recent developments in large language models \citep[][\textit{inter alia}]{chowdhery2022palm, thoppilan2022lamda, liu2023summary} have spotlighted their efficacy in general problem solving \cite{huang2022towards, suzgunChallengingBIGBenchTasks2022}, code generation \cite{chen2021evaluating, austin2021program}, and instruction following \cite{ouyang2022training, bai2022constitutional}. While early models relied on direct answer strategies \cite{brown2020language}, contemporary research has shifted towards linear reasoning paths \cite{wei2022chain, kojima2022large, zhangAutomaticChainThought2022} by breaking problems into sub-tasks for solution discovery, or harnesses external mechanisms to alter token generation by changing the context 
\cite{zhouLeasttoMostPromptingEnables2022, drozdovCompositionalSemanticParsing2022, yaoTreeThoughtsDeliberate2023}.

Analogous to human cognition \cite{sloman1996empirical, kahneman2011thinking}, early LLM strategies seemed to emulate the instantaneous \textit{System 1}, characterized by its impulsive decision-making. In contrast, more recent methodologies like chain-of-thought (CoT) \cite{wei2022chain} and least-to-most prompting (L2M) \cite{zhouLeasttoMostPromptingEnables2022, drozdovCompositionalSemanticParsing2022} reflect the analytical nature of \textit{System 2}. Notably, integrating intermediary reasoning steps has yielded improvements in arithmetic reasoning tasks \cite{srivastava2022beyond, liang2022holistic}.

However, as tasks shift towards deeper planning and extensive thought exploration, these methods appear restrictive. Although CoT integrated with Self-Consistency (CoT-SC) \cite{wangSelfConsistencyImprovesChain2022} enlists multiple LLM outputs for a consensus, the lack of meticulous evaluation can result in model misdirection. The ``Tree of Thoughts'' \cite{yaoTreeThoughtsDeliberate2023, long2023large} emerges as a notable solution. While one LLM is dedicated to idea generation, another steps in to assess the merit of these ideas, following a halting-assessment-resuming cycle. This iterative process, based on a tree search, has shown marked effectiveness, especially in tasks with a breadth of continuations. We see this progression as akin to humans employing tools to circumvent working memory limitations, serving as an external augmentation for LLMs \cite{mialon2023augmented, sel2023learning, gu2024balance}.

\begin{figure*}[t]
\centering
\includegraphics[width=0.9\textwidth]{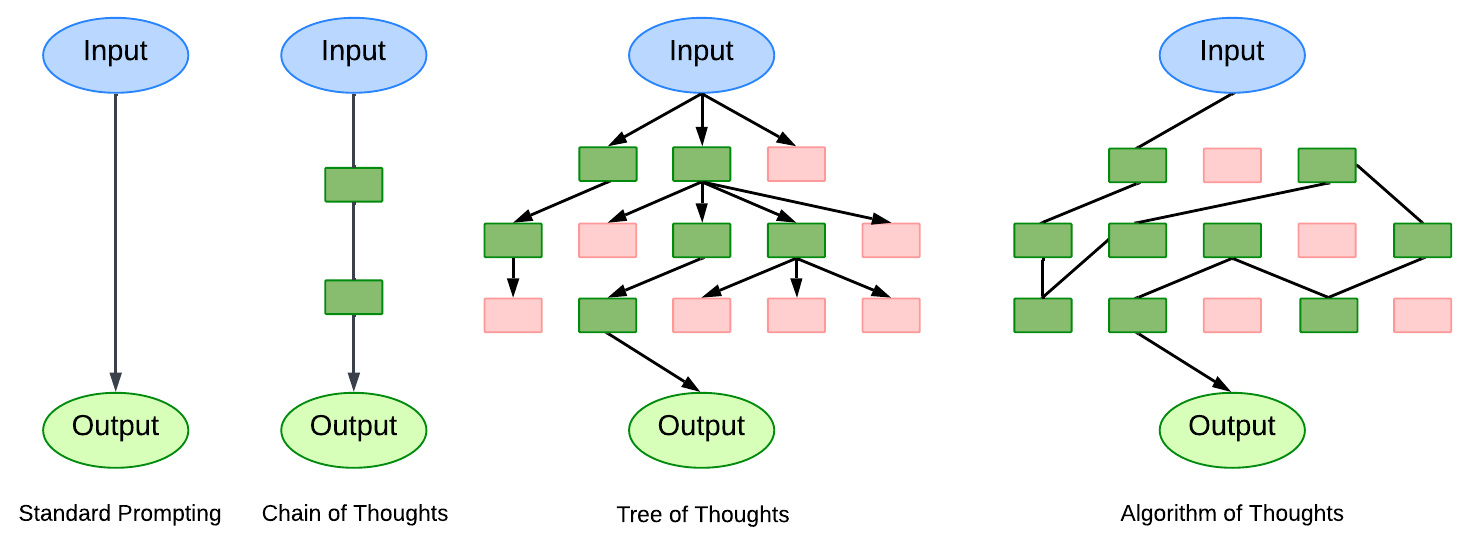}
\caption{Illustration outlining various strategies for tackling reasoning problems with LLMs. Each box signifies a distinct thought, functioning as a unified string of words that forms an incremental pathway to reasoning. Green boxes indicate ideas deemed promising by the LLM, while red boxes represent less promising concepts.}
\label{fig:diff}
\end{figure*}

On the flip side, this enhanced LLM approach is not without pitfalls. A prominent downside is the substantial surge in the number of queries and computational demands. Each query to online LLM APIs such as GPT-4---a focal point of our study---incurs a monetary expense \cite{chen2023frugalgpt} but also contributes to latency, a significant limitation especially critical in real-time applications. Cumulative delays from these queries can compromise solution efficiency. Infrastructure-wise, continuous interactions can stress systems, leading to potential bandwidth constraints and reduced model availability \cite{aminabadi2022deepspeed}. Moreover, the environmental implications cannot be ignored; incessant querying escalates the energy consumption of already power-hungry data centers, exacerbating the carbon footprint \cite{wu2022sustainable,dhar2020carbon,khattar2023winning}. 

With this in mind, our goal was to \emph{dramatically reduce the query counts employed by contemporary multi-query reasoning methods while maintaining performance for tasks necessitating adept use of world knowledge}, thereby steering a more responsible and proficient use of AI resources. Intriguingly, our aim has actually resulted in \emph{surpassing} the performance of such techniques while \emph{requiring significantly fewer tokens for prompting and generation.}

Reflecting on the evolution of LLMs from System 1 to System 2, an essential ingredient comes to light: algorithms \cite{sel2021glsdc, al2023decision, jin2023solution, gu2024safe}.  Characterized by its methodical nature, the algorithmic perspective offers a path to keenly explore problem spaces, enact strategies, and formulate solutions \cite{al2021singular, helie2022psychology, banerjee2022qualitative, sel2022sos, khattar2022cmdp}. While much of the prevailing literature treats algorithms as \emph{external} to LLMs \cite{lincausalprompt}, given LLMs' inherent generative recurrence, can we channel this iterative logic to \emph{internalize} an algorithm?

Drawing upon both the intricate nuances of human reasoning and the disciplined precision of algorithmic methodologies, our work aims to \emph{fuse these two elements to enhance reasoning capabilities} within LLMs. Existing research underscores that humans, when navigating complex problems, instinctively draw upon past efforts, ensuring a comprehensive contemplation rather than a narrow focus \cite{monsell2003task,holyoak2005cambridge,baddeley2003working}. LLMs, with their generative span bounded only by token limits, appear poised to break through the barriers of human working memory. Spurred by this observation, we investigated if LLMs could mirror a similar layered exploration of ideas, referencing prior intermediate steps to sieve out infeasible options, \emph{all within their iterative generation cycle}.  And while humans excel with their intuitive insight, algorithms stand out with organized, systematic exploration. Current techniques, like CoT, often sidestep this synergistic potential, imposing undue pressure on LLMs for on-the-spot precision. {By capitalizing on  LLMs' recursive capabilities, we emulate a hybrid human-algorithmic approach.} This is achieved through our use of algorithmic examples that capture the essence of exploration, from initial candidates to validated solutions. Thus emerges our concept of the \textit{Algorithm of Thoughts} (AoT), as illustrated in Figs. \ref{fig: cot_vs_aot} and \ref{fig:diff}. 

More broadly, our approach signifies a new paradigm of in-context learning. Instead of the traditional ``supervised-learning'' mold of \textsc{[problem, solution]} or \textsc{[problem, successive steps to solution]}, we present a new structure that covers \textsc{[problem, search process, solution]}. Naturally, when instructing an LLM using an algorithm, the anticipation leans towards the LLM simply imitating the algorithm's iterative thinking. However, what emerges as intriguing is the LLM's ability to infuse its own ``intuition'' to achieve a search efficiency that even \emph{surpasses} the algorithm itself (see Fig. \ref{fig: dfs_vs_aot}).

In the subsequent sections, we first situate our work within the existing literature, followed by a discussion of our principal idea. We then present our experimental results and probe a series of hypotheses related to this emerging capability of LLM before rounding off with a conclusion.

\newpage
\section{Related Work}
\noindent\paragraph{Standard Prompting.} Also known as input-output prompting, it provides a few input-output examples of the task before getting an answer for the test sample from the language model \cite{brown2020language}. Although this method is very general and does not need any special prompting strategy, the performance is also worse compared to more advanced methods \cite{shaoSyntheticPromptingGenerating2023, weiEmergentAbilitiesLarge2022, lyu2023faithful}.

\noindent\paragraph{Chain-of-Thought.} In CoT, LLMs are presented with examples where a given question $x$ unfolds through a chain of intermediate reasoning pieces $c_1,\ldots, c_n$ to reach an answer $y$, represented as $x\rightarrow c_1\rightarrow\ldots\rightarrow c_n\rightarrow y$ \cite{wei2022chain, lyu2023faithful}. By mimicking the examples in the context, the LLM automatically divides the solution into simpler linear steps to arrive at the answer, improving performance across numerous reasoning benchmarks. Self-consistency \cite{wangSelfConsistencyImprovesChain2022} is a widely used decoding strategy aimed at generating a variety of reasoning paths by choosing the final answer through a majority vote, though this necessitates additional generations. CoT can be further improved with integrating detailed algorithmic reasoning \cite{zhou2022teaching}. We also utilize algorithmic examples in AoT, however, they are for emerging the inherent heuristic of LLMs to lead the search and not designed to follow a specified pseudocode, or are on language tasks, e.g., creative writing. Contrary to CoT's linear progression, our approach pivots towards the explorative aspect of LLMs. We reconceptualize the $c_1,\ldots, c_n$ sequence, not merely as successive steps towards a solution, but as a dynamic, potentially mutable path that resembles an algorithmic search, allowing for exploration, recalibration, and non-linear progression.

\noindent\paragraph{Least-to-Most prompting (L2M).} Taking cues from educational psychology \cite{libby2008comparison}, L2M prompting directs the LLM to decompose the central problem into smaller subproblems. Each subproblem is tackled in sequence, with the outcome appended before progressing to the next  \cite{zhouLeasttoMostPromptingEnables2022, drozdovCompositionalSemanticParsing2022}. While this structured delineation is beneficial for broader generalization, it operates on the premise of finding a nearly perfect decomposition in a single attempt---ideal for problems with a clear-cut structure. Yet, when tasks intertwine with their decomposition complexities (like games of 24), this method's inflexibility becomes apparent.  Contrastingly, AoT not only underscores the active subproblem (as shown in Fig. \ref{fig: cot_vs_aot}), but also permits a more contemplative approach by entertaining various options for each subproblem, while maintaining efficacy even with minimal prompts.

\noindent\paragraph{Tree of Thoughts (ToT).} In the cases where each subproblem has multiple viable options to explore, linear reasoning paths from CoT or L2M substantially limit the coverage of the thought space. Considering possible options for each subproblem, the decision tree can be explored by external tree-search mechanisms (e.g., BFS, DFS) \cite{yaoTreeThoughtsDeliberate2023, jin2023human, sel2024skin}. Evaluation capabilities of LLMs can also be used to direct the search by pruning nodes that are hopeless to increase efficiency. However, ToT, due to its requirement for multiple queries to the LLM for a solution, demands significantly more computation than AoT. Additionally, it necessitates evaluating the potential of each search node in the in-context examples and writing specialized functions to extract information from model responses to maintain the tree structure externally. In stark contrast, AoT requires just a single prompt and no coding skills, greatly democratizing LLM use for complex problems.

\section{Algorithm of Thoughts}
Our strategy pivots on recognizing \textit{a core shortcoming} of current in-context learning paradigms. CoT, while enhancing the coherency of thought linkages leading to solutions, occasionally falters, presenting incorrect intermediate steps \cite{zelikman2022star, turpin2023language, lanham2023measuring}. Faithful CoT \cite{lyu2023faithful} ought to amend this by eliciting symbolic chains of reasoning where the LLM's output resembles task-specific pseudo-code, primed for deterministic execution like Python. The intention is only to use the thought processes but not the outputs and inputs of each link since they have a tendency to be unreliable. 
But, \textit{the occasional missteps of CoT may not necessarily be due to the LLM's inability to compute correctly}. The LLM, when confronted with questions that closely match conditions of previous in-context examples, may favor {echoing those outputs over generating the appropriate questions}. To shed light on this phenomenon, we designed an experiment. 

Querying \textit{text-davinci-003} for arithmetic tasks (e.g., `$11-2=$'), we prefixed them with multiple in-context equations converging to an identical output (e.g. `$15-5=10, 8+2=10$'). Our results, presented in Fig. \ref{fig: simple_insight}, reveal a steep decline in accuracy, suggesting that {the mere presence of correct reasoning in the context might inadvertently compromise even basic arithmetic skills}.

\begin{figure}[h]
\centering
\includegraphics[width=0.4\textwidth]{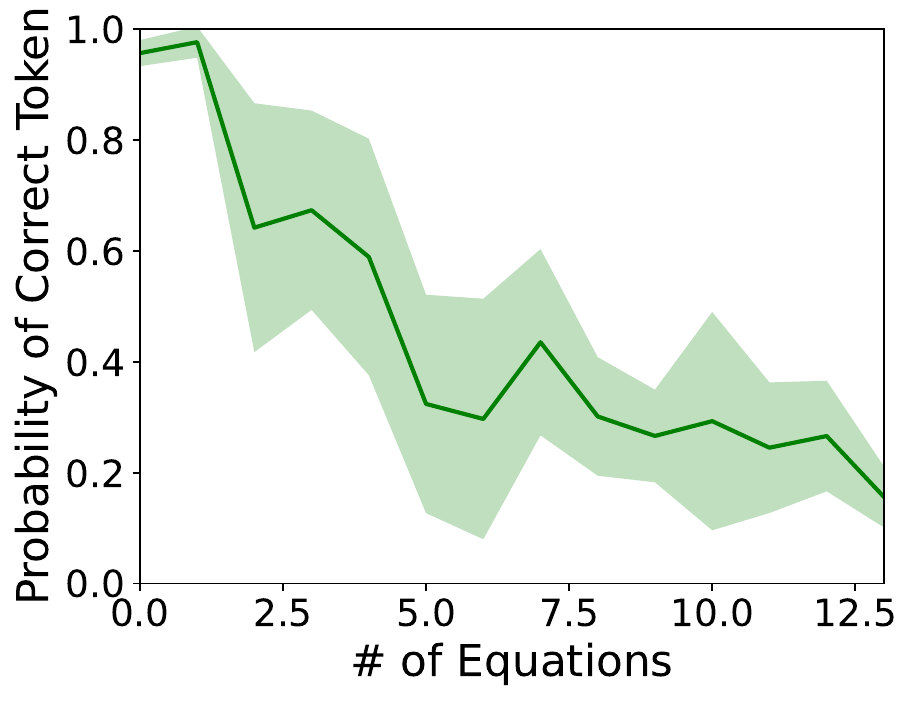}
\caption{The probability of generating the correct token as we add more in-context examples that are correct but possess identical outputs.}
\label{fig: simple_insight}
\end{figure}

To offset this bias, diversifying the outputs of examples might seem like a viable solution, but this could subtly skew the distribution of outputs. Merely adding unsuccessful trials, much like a random search, might inadvertently encourage the model to retry rather than truly solve. Capturing the true essence of algorithmic behavior, where both failed searches and subsequent recovering and learning from such attempts play a role, we incorporate \emph{in-context examples patterned after search algorithms}, notably depth-first search (DFS) and breadth-first search (BFS). See Fig. \ref{fig: cot_vs_aot} for an example. 

This paper focuses on a broad class of tasks reminiscent of tree-search problems. These tasks necessitate breaking down the main problem, crafting feasible solutions for each segment, and making decisions on the paths to either pursue or forsake, with the option of reevaluating more promising segmentations. Rather than posing separate queries for every subset, we leverage the iterative capabilities of the LLM to address them in one unified generation sweep. By confining ourselves to one or two LLM interactions, this approach naturally incorporates insights from antecedent context candidates and tackles intricate issues requiring an in-depth exploration of the solution domain. 
We also give insights into how small or big those thoughts should be and what type of in-context examples should be given to the LLM to promote token efficiency. Subsequently, we outline key components of tree-search algorithms and their manifestation in our framework.

\paragraph{1. Dividing the search into steps.}
Similar to creating step-by-step solutions in CoT or L2M, we also need to identify intermediate search layers. This is akin to creating examples for CoT, especially for tree-search problems, where the correct reasoning path resembles a CoT solution. The challenge lies in selecting the right chain from numerous candidates at each layer to reach the final answer. Thus, our focus will be more on generating the search process for in-context examples rather than how to solve each subproblem after selecting the next chain.

\begin{figure}[h]
\centering
\includegraphics[width=0.3\textwidth]{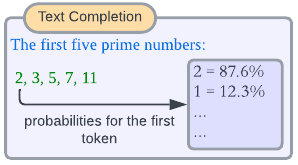}
\caption{An example highlighting the drawback of isolated sampling of sequenced ideas.  Input is denoted in blue, with the \textit{text-davinci-003} providing the green completions.}
\label{fig: simple_alignment_issue}
\end{figure}

\paragraph{2. Proposing Solutions to Subproblems.} A dominant approach in existing works involves direct sampling from LLM token output probabilities \cite{wangSelfConsistencyImprovesChain2022, yaoTreeThoughtsDeliberate2023}. Though effective for one-off answers \cite{kadavath2022language} 
, this method falls short in scenarios demanding a sequence of samples to be integrated or evaluated within subsequent prompts \cite{robinsonLeveragingLargeLanguage2022}. 
To minimize model queries, we adopt an uninterrupted solution-creation process.  Here, we directly and continuously generate solutions for the prevailing subproblem \emph{without} any generation pauses.

The benefits are three-fold. First, with all generated solutions existing within a shared context, there is no need for individual model queries for each solution evaluation. Second, while it may seem counterintuitive, isolated token or token group probabilities might not always yield meaningful choices. A simple illustration is found in Fig. \ref{fig: simple_alignment_issue}. When evaluated independently, the second-most probable token for our inaugural number is `$1$'---not qualifying as prime. But, when generation remains unbroken, the derived sequence is correct. This incongruence points towards the restrictive nature of the Markov property in sequence modeling.  Core to our perspective is the premise that for sequential tasks like algorithmic search, LLMs are more adept at generating entire sequences than intermittently pausing and re-initiating the token sampling process.

\paragraph{3. Evaluating the Promise of a Subproblem.} 
 Existing techniques lean on additional prompting to discern the potential of tree nodes, aiding decisions regarding exploration direction. Our observations suggest that if the most promising routes are encapsulated within the in-context examples, LLMs inherently gravitate towards prioritizing those promising candidates. This diminishes the need for intricate prompt engineering and allows the incorporation of intricate heuristics, whether intuitive or knowledge-driven. Again, the absence of disjoint prompts in our approach allows for an immediate assessment of candidate viability in the same generation.

\paragraph{4. Backtracking to a More Promising Node.} The decision of which node to explore next (including retracing to a prior node) inherently depends on the selected tree-search algorithm.
While previous studies \cite{yaoTreeThoughtsDeliberate2023} have employed external means, such as coded mechanisms for the search process, this restricts its broader appeal and entails additional customization. Our designs predominantly adopt a DFS approach supplemented by pruning.  The aim is to maintain proximity between nodes sharing the same parent, thereby encouraging the LLM to prioritize local over distant features. Additionally, we present performance metrics for the AoT approach grounded in BFS. Our reliance on the model's inherent capacity to glean insights from in-context examples obviates the necessity for additional mechanisms.

\paragraph{Expressiveness of LLMs with AoT.}
Recent works have investigated the expressivity of transformers with standard and CoT prompting \cite{chiang2023tighter, schuurmans2023memory, merrill2023expresssive, feng2023towards}. We provide the following theoretical result for AoT, which implies that it can tackle NP problems, extending from P problems that of COT's.

\begin{corollary}[Informal]
    Consider \textsc{TIME}$(a^n)$ as the class of problems for which a Turing machine exists that operates within a time complexity of $\mathcal{O}(a^n)$ for some $a\geq 1$. If a transformer can generate $a^n$ intermediate tokens to solve the problem when prompted by AoT, we have
    \begin{equation}
        \textsc{TIME}(a^n) \subseteq \textsc{AoT}(n),
    \end{equation}
    where $\textsc{AoT}(n)$ refers to the decoding steps by AoT when the input has $n$ tokens.
    \label{corollary: main}
\end{corollary}
The proof of the above corollary is given in the appendix.

\section{Experiments}
We show that AoT surpasses the performance of other single-prompt methods (e.g., standard, CoT/-SC prompting) and even that of strategies utilizing external mechanisms, such as ToT, across the benchmarks we tested. We present the results for the creative writing task in the appendix. In addition, we show that AoT continues to have an advantage over standard prompting or CoT even after fine-tuning. This implies that the issue with LLMs is not simply a minor misalignment or a deficiency in domain expertise. Rather, it underscores the necessity of AoT prompting. This approach is vital because the nature of the tasks we evaluated inherently demands a thorough exploration of solution paths, a requirement that goes beyond simple fine-tuning adjustments. For the generation of in-context examples, we have asked the authors to write down their search process and randomly chosen from that list. We have written them again in a simple structured way to have uniformity between the examples to create our AoT prompts. More details regarding this process for each task is given in the AoT setup subsections.

\subsection{Game of 24}
The game of 24 is a mathematical card game in which players are given four numbers and must use addition, subtraction, multiplication, and division (each operation can be used more than once) to manipulate those numbers to reach a total of 24. For instance, for the numbers `$8\:8\:5\:4$', one solution could be `$8 * (5 - (8 / 4)) = 24$'. At first glance, the game might appear straightforward. However, a cursory calculation suggests there are nearly 13,000 distinct expressions possible for any set of four numbers, making it a formidable challenge for present-day LLMs.

\paragraph{Task Setup.} Adhering to the setup detailed in \cite{yaoTreeThoughtsDeliberate2023}, we use games from indices 901-1000, sourced from the 1362 games ranked by relative difficulty at \textit{4nums.com}.  An attempt is considered successful if it is able to reach a total of 24 using the exact numbers provided and only the allowed operations.

\paragraph{Baselines.} Standard prompting and CoT are used in the 5-shot setting, with CoT integrating 3 steps for the operations. These methods are sampled 100 times, and the averaged success rates from these samples are reported. CoT-SC is also tested with 100 votes in our setup. For ToT, we use a breadth of 5. 

\paragraph{AoT Setup.} 
We employ the same 5-shot setting as in standard prompting and CoT baseline setup. Our in-context samples leverage a DFS-style search algorithm, which is the same version used when contrasting with traditional DFS in Fig. \ref{fig: dfs_vs_aot}. During each subtree exploration, dubbed either the `first step' or  `first operation', we choose two numbers---illustrated by the selection of 8 and 6 in the third 'first step' (i.e., subtree labeled `3') of Fig. \ref{fig: cot_vs_aot}---and a corresponding operation (e.g., $8-6$).  This operation results in a new number, 2, leaving us with three numbers in total. A thorough combing of these three numbers culminates in 19 leaf nodes, all visible under the `3' subtree in Fig. \ref{fig: cot_vs_aot}. In order to generate our in-context examples, we have randomly selected games that do not appear at test time. We asked the authors to write the search steps they used until they arrived at the answers. These are exactly the node selection, node expansion steps with inherent heuristics of the individuals. Then, we have selected randomly from these solutions and written them in a trivial structured way to assure uniformity between the examples. The exact prompts we use are given in the Prompts section under the `AoT (DFS)' subsection in the appendix.  We aim to assess two aspects: the ability of the LLM to pinpoint promising first operations, which directly impacts the number of resolved leaf nodes, and its performance against a conventional DFS. Details on the prompts are provided in the appendix. As our method emphasizes sequential generation over trajectory sampling, we operate with a temperature setting of 0.

\paragraph{Results.} From Table \ref{table: game_of_24_main_results}, it is evident that standard prompting combined with CoT/-SC significantly lags behind tree search methods when used with LLMs. The ``Standard + Refine'' result, showing a 27\% success rate, is referenced from \cite{yaoTreeThoughtsDeliberate2023}. This method involves iteratively asking the LLM (up to 10 iterations) to refine its answer if the initial one is incorrect. Meanwhile, ToT is limited to a maximum of 100 node visits, translating to several hundred LLM queries for each example. Remarkably, AoT achieves its results with just a \textbf{single query!} Despite reducing the number of requests by more than a factor of 100, AoT still outperforms ToT in this task. Furthermore, AoT is also more efficient than ToT in terms of the total number of prompt tokens given to the LLM and the completion tokens it generates. 

\renewcommand{\arraystretch}{1.2} 
\begin{table}[h]
\centering
\begin{tabular}{lllll}
\hline
\textbf{Method}    & \textbf{Success} & \textbf{Queries} & \textbf{PTs} & \textbf{CTs} \\ \hline
I/O & $7.3\%$          & $1$            & $164$          & $18$           \\
CoT                & $4.0\%$         & $1$            & $421$          & $46.2$         \\
CoT-SC   & $9.0\%$          & $100$            & $42$,$100$        & $4$,$620$       \\
I/O + Refine  & $27\%$           & $10$             & $458$          & $360$          \\ \hline
ToT $(b=5)$        & $69\%$           & $109.1$          & $13$,$900$        & $5$,$500$         \\
AoT (ours)         & $\mathbf{71}\boldsymbol{\%}$         & $1$            & $5$,$450$         & $998.4$        \\ \hline
\end{tabular}
\caption{Game of 24: success rates and the average number of LLM queries for each example. We give the average query count, prompt tokens (PT), and completion tokens generated by the LLM (CT).}
\label{table: game_of_24_main_results}
\end{table}

\paragraph{Error Analysis.}
Using a strictly LLM-centric approach---eschewing any external tooling or edits---we sought to categorize mistakes observed during the game of 24. This aids in highlighting areas for refinement when solely deploying LLMs. We've classified these errors into four distinct categories:
\textbf{1)} Out-of-token error: The LLM reaches its maximum token threshold without identifying a solution.
\textbf{2)} Expression misstep: The LLM has the correct logic or steps but fails when trying to express or formulate them into a coherent answer.
\textbf{3)} Non-finalization error: The LLM discovers the solution but continues its search without consolidating the finding.
\textbf{4)} Other errors: This umbrella term encompasses other mistakes like computational errors that result in overlooking the solution or furnishing incorrect answers. To exclusively showcase the AoT's search capabilities, we also present the AoT + Manual Resolution version. Here, once the LLM pinpoints a solution, its final articulation is manually processed---a strategy also employed by the ToT method. As evidenced in Table \ref{table: game_of_24_error_analysis}, a notable 7\% of mistakes stem from non-algorithmic factors like non-finalization and expression missteps.  In fact, with manual resolution, AoT attains a 78\% success rate, surpassing ToT. This underlines the potential for refining our prompt, especially in areas concerning recognizing and expressing successful problem resolutions. Additionally, the token limitation underscores the appeal of expanding the generative context window, which may further bolster LLMs' recursive reasoning when engaged with algorithmic examples.

\renewcommand{\arraystretch}{1.2} 
\begin{table}[h]
\centering
\begin{tabular}{ll}
\hline
\textbf{Error Type} & \textbf{Error}   \\ \hline
Out-of-token error         & $9\%$              \\
Expression misstep         & $4\%$              \\
Non-finalization error   & $3\%$              \\
Others              & $13\%$             \\ \hline
\textbf{Method}     & \textbf{Success} \\ \hline
ToT                 & $69\%$             \\
AoT                 & $71\%$             \\
AoT + Manual Resolution  & $78\%$             \\ \hline
\end{tabular}
\caption{Game of 24: AoT error analysis.}
\label{table: game_of_24_error_analysis}
\end{table}

\subsection{Mini Crosswords}
The $5\times 5$ mini crossword is a compact word puzzle featuring a grid of 25 squares arranged in a $5$-by-$5$ configuration. Players are tasked with filling the grid based on provided clues for each word. Clues are given for words that run both across (horizontally) and down (vertically). Words intersect at certain letters, offering additional hints to complete the puzzle.

\paragraph{Task Setup.}
Adhering to the setup outlined in \cite{yaoTreeThoughtsDeliberate2023}, we draw our prompts from games 136, 141, 146, 151, and 156 out of the 156 games available on \textit{goobix.com}. Our testing focuses on a set of 20 games, specifically games 1, 6, $\ldots$, 91, and 96.

\paragraph{Baselines.}
As done in the game of 24, we benchmark our method against established techniques: standard prompting, CoT, and ToT. For standard prompting, we provide both the crosswords and their respective solutions as in-context examples.  CoT augments this by prompting the retrieval of words for each of the ten clues---equally split between horizontal and vertical orientations. We directly extract the success rates of ToT from their paper for comparison.

\paragraph{AoT Setup.} 
We divide the process into two steps, each involving a query. Initially, we task the LLM with suggesting five potential words for each row and column. We then pinpoint the starting word candidates that have the highest compatibility with other words within the crossword framework.  This preliminary phase mirrors a 'warm-up' sequence in algorithm initialization. In the subsequent step, we exclusively leverage the LLM's algorithmic reasoning prowess, starting with the pre-selected word. The method involves cyclically choosing a likely option 
for insertion, generating candidate words, and assessing their compatibility with the words already on the board. If no match is found, the process shifts focus to another promising candidate. Otherwise, the word is added to the crossword, and the search continues. The cycle concludes either when the board is fully populated or no more suitable words can be found, which may be due to either incorrect existing words or the absence of matching words. Notably, this entire process unfolds within a single-generation window. The algorithmic examples in our prompt (detailed in the Appendix) include three that achieve game completion and two that predominantly populate the crossword, filling 8 or 9 slots.

\paragraph{Results.} Table \ref{table: minicrosswords_main_results} underscores AoT's proficiency in the mini crosswords task, showcasing a word success rate---a measure used in existing studies to represent the percentage of words correctly completed out of the total---that surpasses earlier methods reliant on various prompting techniques. It also outperforms ToT. An important observation is the sheer volume of queries ToT employs, exceeding AoT's by over a factor of 100. AoT also enjoys 25x reduction in total tokens required compared to ToT, a benefit of having everything in-context. 

\renewcommand{\arraystretch}{1.2} 
\begin{table}[h]
\centering
\begin{tabular}{lllll}
\hline
\textbf{Method} & \textbf{W. Success} & \textbf{Queries} & \textbf{PTs} & \textbf{CTs} \\ \hline
I/O             & $14\%$                & $1$              & $790.3$        & $30.5$         \\
CoT-SC          & $15.6\% $             & $1$              & $1$,$400$         & $1$,$600$         \\
ToT             & $46.5\%$              & $>200$           & $96$,$700$        & $21.8$k        \\
AoT (ours)      & $\mathbf{52}\boldsymbol{\%}$                & $2$              & $3$,$800$         & $975.6$        \\ \hline
\end{tabular}
\caption{$5\times 5$ mini crosswords word: word success rates and the average number
of LLM queries for each example. We give the average query count, prompt tokens (PT), and completion tokens generated by the LLM (CT).}
\label{table: minicrosswords_main_results}
\end{table}

\paragraph{Error Analysis.} To understand the prevalent mistakes made by AoT, we've categorized the errors into four distinct categories. In our analysis for each game, we focus on the initial error the LLM produces while charting its reasoning path, given that an early error typically cascades into subsequent failures. \textbf{1)} No preselections: LLM fails to generate compatible words essential for the warm-start phase. Given a correctly preselected word, the second phase for recursive reasoning can exhibit errors including: \textbf{2)} Expression misstep: The LLM mistakenly believes it has exhausted all choices and jumps to an answer prematurely. \textbf{3)} Incorrect pattern extraction: The LLM wrongly extracts a pattern based on the current board layout. \textbf{4)} Erroneous word placement: Despite recognizing the correct pattern, the LLM selects a mismatched word or misses better-fitting alternatives. Navigating the crossword complexity arises from outdated terms and esoteric references.
Predominantly, the errors observed are due to misguided word placements followed by pattern misinterpretations. Also, the LLM seems challenged in aligning letters at precise indices to create word structures— an obstacle circumvented by an external mechanism in the ToT framework.

\renewcommand{\arraystretch}{1.2} 
\begin{table}[h]
\centering
\begin{tabular}{ll}
\hline
\textbf{Error Type}         & \textbf{Error} \\ \hline
No preselections            & $15.8\%$         \\
Expression misstep          & $5.3\%$         \\
Incorrect pattern extraction & $26.3\%$         \\
Erroneous word placement   & $52.6\%$         \\ \hline
\end{tabular}
\caption{Breakdown of errors in $5\times 5$ mini crosswords with AoT. Numbers indicate the relative percentage of each error type among all errors.
}
\label{table: crossword_error_analysis}
\end{table}

\subsection{Finetuning}
In order to eliminate the possibility that prior experiments lacked domain knowledge or were misaligned with the task, even after few-shot prompting via standard prompting or CoT, we also finetuned GPT-3.5-Turbo using OpenAI's API with 900 examples with CoT and AoT. In Table \ref{table: finetuning}, we can see that although GPT-3.5-Turbo had similar solution rates for the Game of 24 with CoT and AoT, AoT fine-tuning improved the model by 60\% compared to 8\% for CoT. This shows that fine-tuning alone cannot emerge implicit non-linear thinking, and LLMs still require explicit exploration of possible options to arrive at a solution. This is similar to chess grandmasters being able to find better moves than others even when they play without thinking deeply. However, to find the truly great moves, they are also required to deliberately explore the possibility of space.

\begin{table}[h]
    \centering
    \begin{tabular}{lll}
    \hline
    \textbf{Method} & \textbf{w/o finetuning} & \textbf{w/ finetuning} \\ \hline
    CoT             & 3\%                     & 12\%                   \\
    AoT             & 3\%                     & \textbf{63\%}          \\ \hline
    \end{tabular}
    \caption{\textit{AoT's advantage continues even after finetuning.} Finetuning results on the Game of 24 on 900 examples with CoT and AoT prompting.}
    \label{table: finetuning}
\end{table}

\section{Discussion}
In this section, we delve into crucial aspects to consider when crafting prompts for AoT, using the game of 24 as our primary case study.

\paragraph{Can AoT surpass the DFS it is patterned after?} A core query of ours is to ascertain if the LLM has the capability to not only mirror but also outdo the efficiency of the algorithm introduced in-context. As evidenced in Fig. \ref{fig: dfs_vs_aot}, AoT systematically navigates fewer nodes than its DFS counterpart. While DFS employs a uniform strategy when choosing the subsequent subtree to investigate, AoT's LLM integrates its inherent heuristic.  This amplification over the base algorithm exemplifies the advantages of LLM's recursive reasoning capability.

\begin{figure}[h]
\centering
\includegraphics[width=0.45\textwidth]{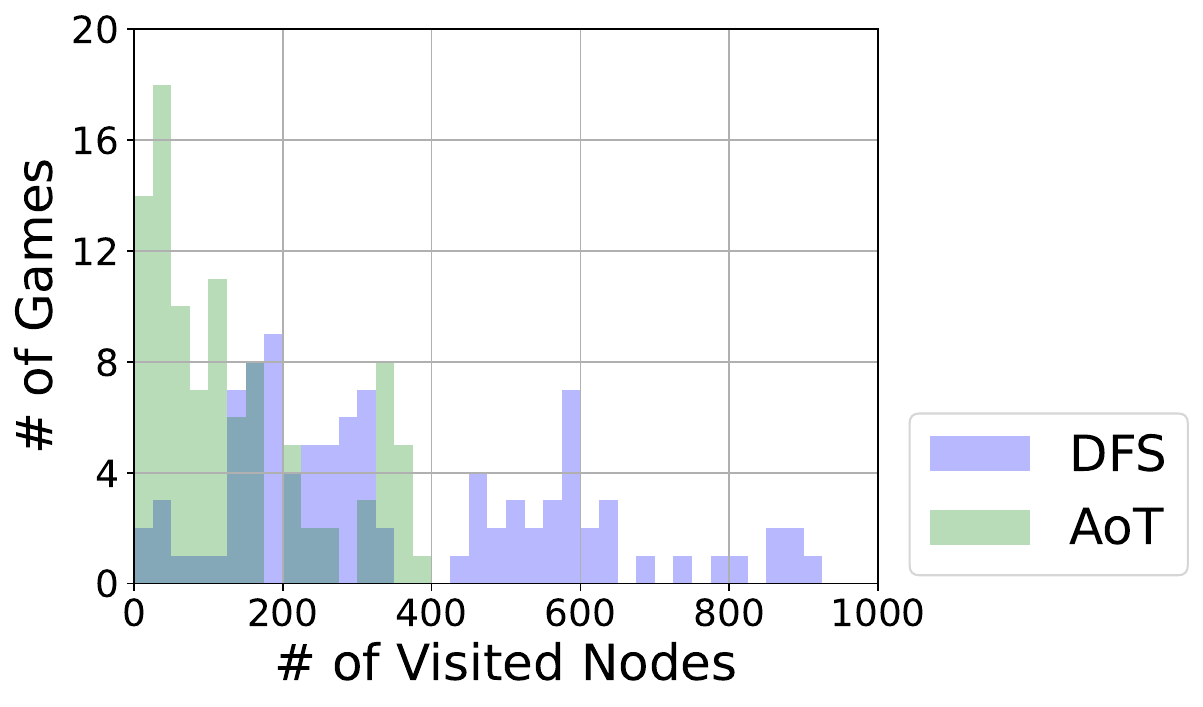}
\caption{Histogram showing the number of visited nodes for AoT and DFS in the Game of 24.}
\label{fig: dfs_vs_aot}
\end{figure}

\paragraph{How does the search step count within the algorithmic example modulate AoT's behavior?} We begin with the standard AoT prompt and modify the subtree explorations. In AoT (Short), each in-context example uses one or two steps to reach a solution, while AoT (Long) incorporates three to five extra subtree explorations. The impact on total search steps is illustrated in Fig. \ref{fig: short_vs_normal_vs_long}. Our observations highlight longer generations for AoT (Long) and shorter ones for AoT (Short) relative to the original AoT. This suggests that the search step count introduces an implicit bias on the LLM's search velocity. Notably, even when navigating incorrect steps, it's essential to emphasize the exploration of promising directions. 

\begin{figure}[h]
\centering
\includegraphics[width=0.43\textwidth]{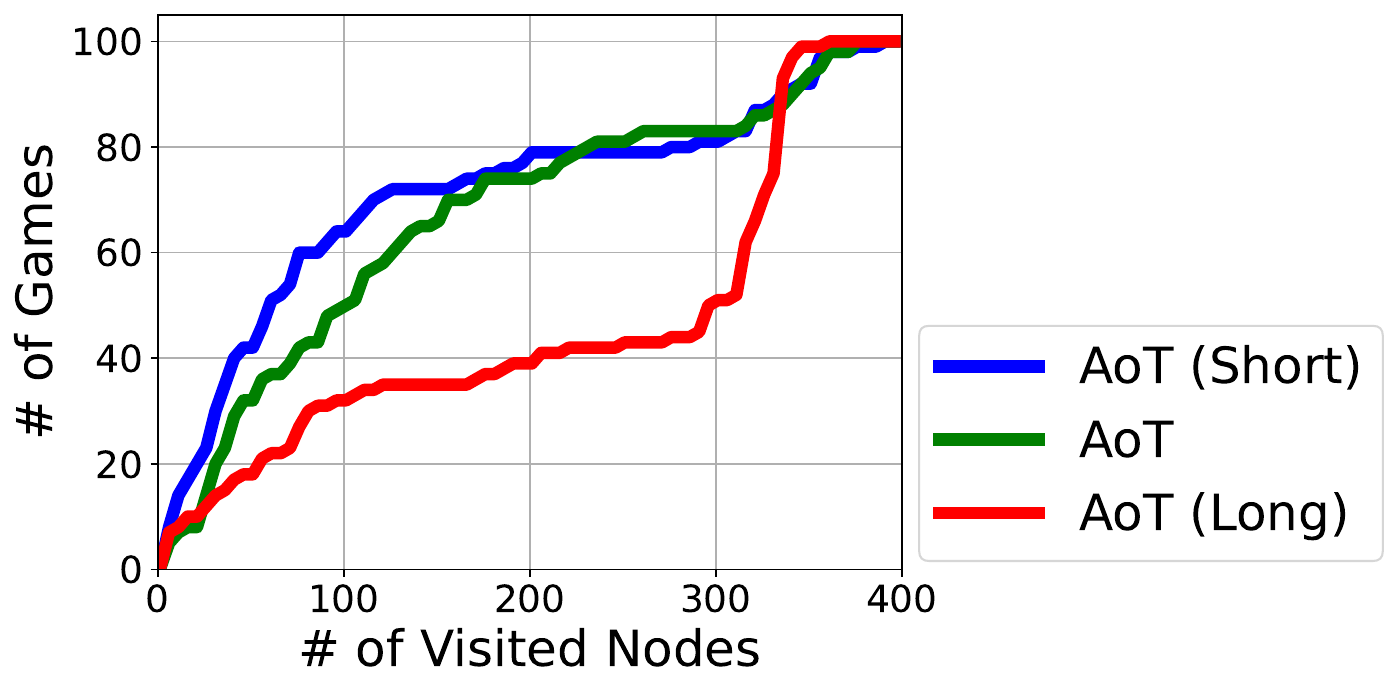}
\caption{Comparison of AoT with shorter and longer in-context examples prompted AoT versions: cumulative number of games for the number of visited nodes. }
\label{fig: short_vs_normal_vs_long}
\end{figure}

\paragraph{Can AoT be used for question-answering tasks?} To answer this question, we have followed the same structure to the Creative Writing task (given in the appendix of our paper) to evaluate AoT and the baselines on the first 100 questions of well-known GSM8K and StrategyQA benchmarks. Briefly, we implemented a zero-shot AoT prompt for StrategyQA and GSM8K that proposes 3 strategies and expands them with detail to select the best one. As seen in Table \ref{tab: qa}, we see a slight boost on this task due to GPT-4 with CoT already being competent.

\begin{table}[h]
    \centering
    \begin{tabular}{ccc}
    \hline
    \textbf{Method} & \textbf{GSM8K} & \textbf{StrategyQA} \\ \hline
    IO              & 51\%           & 73\%                \\
    CoT             & 86\%           & 82\%                \\
    ToT             & 90\%           & 83\%                \\
    AoT             & 89\%           & 84\%                \\ \hline
    \end{tabular}
    \caption{Performance comparison of different methods on question-answer tasks using GSM8K and StrategyQA benchmarks. The AoT model shows competitive performance, especially when compared with the CoT and ToT methods.}
    \label{tab: qa}
\end{table}

\paragraph{Can AoT work as other dynamic programming methods?}
We have also tested AoT and the baselines on the traditional dynamic programming problems Coin Change and Edit Distance, where DFS and BFS have explosive complexities. However, another DP method named tabulation can more easily solve these problems. Since ToT cannot take the form of tabulation, and has to either DFS or BFS, we decided not to include those poor results to be fair. However, one can use a single leaf node with a CoT prompt to solve these problems. There, ToT's performance can be considered the same as that of CoT's. Please refer to Table \ref{tab:dp_performance} for the detailed results.

\begin{table}[h]
\centering
\begin{tabular}{ccc}
\hline
\textbf{Problem} & \textbf{Coin Change} & \textbf{Edit Distance} \\ \hline
I/O              & 72\%                 & 61\%                   \\ 
CoT              & 76\%                 & 64\%                   \\ 
AoT              & 96\%                 & 90\%                   \\ \hline
\end{tabular}
\caption{Performance comparison of AoT with traditional dynamic programming methods on solving Coin Change and Edit Distance problems.}
\label{tab:dp_performance}
\end{table}

\paragraph{Can AoT help other SOTA LLMs?}
We investigate AoT for other SOTA LLMs, Claude 3 and Gemini 1.5 Pro, to see if it provides a significant boost for them as well on the game of 24. We see that both Claude 3 and Gemini 1.5 Pro benefit significantly. We were unable to run Gemini 1.5 Pro on CoT-SC due to API access not being available yet. Please refer to Table \ref{tab:model_results} for detailed results.

\begin{table}[h]
\centering
\begin{tabular}{cccc}
\hline
\textbf{Method} & \textbf{GPT-4} & \textbf{Claude 3} & \textbf{Gemini 1.5 Pro} \\ \hline
IO              & 7\%            & 6\%               & 6\%                     \\
CoT-SC          & 9\%            & 9\%               & -                       \\
AoT             & 71\%           & 68\%              & 55\%                    \\ \hline
\end{tabular}
\caption{Additional language model results for the AoT, CoT-SC, and IO methods across different models.}
\label{tab:model_results}
\end{table}

\section{Conclusion}
This paper presents the \emph{Algorithm of Thoughts}, a pioneering prompting strategy to navigate reasoning pathways in LLMs using minimal queries. Our findings reveal that this method not only substantially surpasses prior single-query techniques but also outperforms external tree-search implementations. Such an approach augments the potential to streamline idea discovery in LLMs, balancing both cost and computational demands. Future work includes designing token-efficient algorithmic examples, developing adaptive mechanisms for ``tunnel-vision'' activation to expedite the search, and deepening the understanding of this fresh mode of in-context learning from theoretical angles.

\section{Limitations} While AoT substantially cuts down on the number of queries relative to ToT, its resource demands exceed those of standard prompting and CoT, a consequence of its extensive exploration of ideas via token generation. Crafting token-efficient algorithmic examples is one direction of future research. It is also pertinent to highlight that we conducted our tests exclusively with GPT-4. Though more costly than other LLMs, GPT-4's advanced capabilities appear pivotal for AoT's optimal functioning; models of lesser caliber might not yield comparable performance boosts from AoT.

\newpage
\section*{Acknowledgments}
This work was supported in part by the Amazon Research and Virginia Tech Initiative for Efficient and Robust Machine Learning and the National Science Foundation (Grants \#2331775 and \#2312794).

\section*{Impact Statement}
This paper presents work whose goal is to advance the field of Machine Learning. There are many potential societal consequences of our work, none which we feel must be specifically highlighted here.

\bibliography{references}
\bibliographystyle{icml2024}

\newpage
\appendix
\onecolumn
\section{Game of 24 - Additional Details}
In order to avoid confusion in our analysis of AoT in the game of 24, we give additional details in terms of terminologies we use as well as their direct implications in the performance figures. An Illustration of these are given in Fig. \ref{fig: illustration_definition}.

\begin{figure}[h]
\centering
\includegraphics[width=1.0\textwidth]{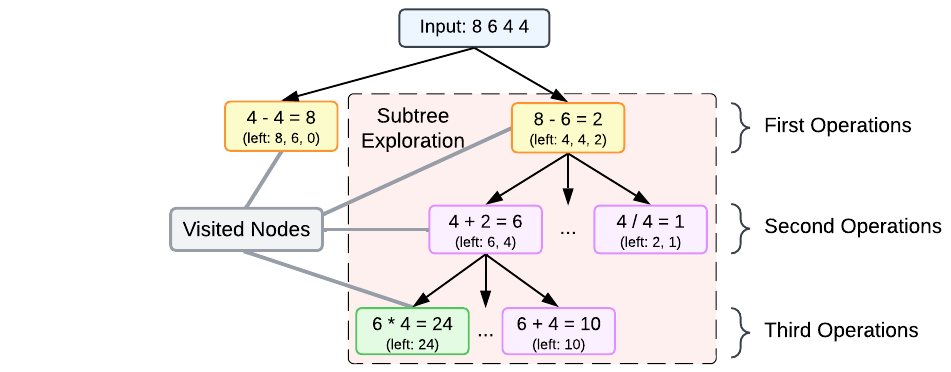}
\caption{An illustration of terminologies we use for the game of 24. The yellow nodes represent the first operations and the states they lead to; the green node represents the node where we find the solution; all other nodes are represented by pink.}
\label{fig: illustration_definition}
\end{figure}

\paragraph{First operations / First iterations.} This represents the scenario that after we choose the first two number in the game of 24, the case of either adding, subtracting, multiplying or dividing them.

\paragraph{Subtree Exploration. } This denotes searching all or most of the nodes coming from the same state, typically states with less than four numbers left.

\paragraph{Number of nodes visited.} This is the number of states that the method has been on the game of 24. Each state is the set of number we are left with, after our operations in the numbers. For example, after the first operation we might be left with the numbers `$8\:3\:1$'. This set of numbers represent a state, as well as the state of `$8\:3$' that we will be left with after another operation of `$8*1=8$'.

\section{Creative Writing}
We use the creative writing task, also used by \cite{yaoTreeThoughtsDeliberate2023}, where the LLM is provided with four arbitrary sentences. The objective is to craft a cohesive narrative divided into four paragraphs, with each paragraph culminating in one of the given sentences. This exercise not only fosters creativity but also emphasizes strategic deliberation.

\subsection{Task Setup}
Sentences are randomly sourced from \textit{randomwordgenerator.com}, resulting in 100 distinct sets of inputs. Given the absence of predetermined correct answers, the primary focus lies in evaluating the coherence of the responses. We have noted that GPT-4 consistently aligns with these input guidelines. Evaluation is centered around assessing passage coherence using a GPT-4 zero-shot prompt, where each output is rated on a scale of 1 to 10. Each task response undergoes five such evaluations, with their scores being averaged subsequently.

\subsection{Baselines}
For this task, both standard and CoT prompts are employed without preliminary training. While the standard prompt directly guides the LLM to fashion a cohesive narrative based on stipulated parameters, the CoT prompt obliges the model to initially outline a succinct plan prior to drafting the narrative, serving as an intermediate cognitive bridge. For each task iteration, ten samples are generated using both the standard and CoT methods. Results of the ToT approach are presented without modification.

\subsection{AoT Setup}
Mirroring ToT's methodology, the task is tackled in a zero-shot setting. Our prompt instructs the model to first formulate five distinct plans. Subsequent to this, the model selects the most promising among them to shape a narrative and then refines it for optimal coherence. The exact prompts used for this zero-shot approach will be provided in the subsequent section.

\subsection{Results}
As depicted in Fig. \ref{fig: text_performance}, AoT outpaces other singular query prompting techniques such as standard prompting and CoT in terms of performance. It also exhibits a marked improvement over ToT, although the difference is not statistically significant. Comprehensive scores, along with the average query count needed for each method, are consolidated in Table \ref{table: creative_writing_results}. Notably, AoT necessitates fewer queries compared to ToT.

\begin{figure}[h]
\centering
\includegraphics[width=0.45\textwidth]{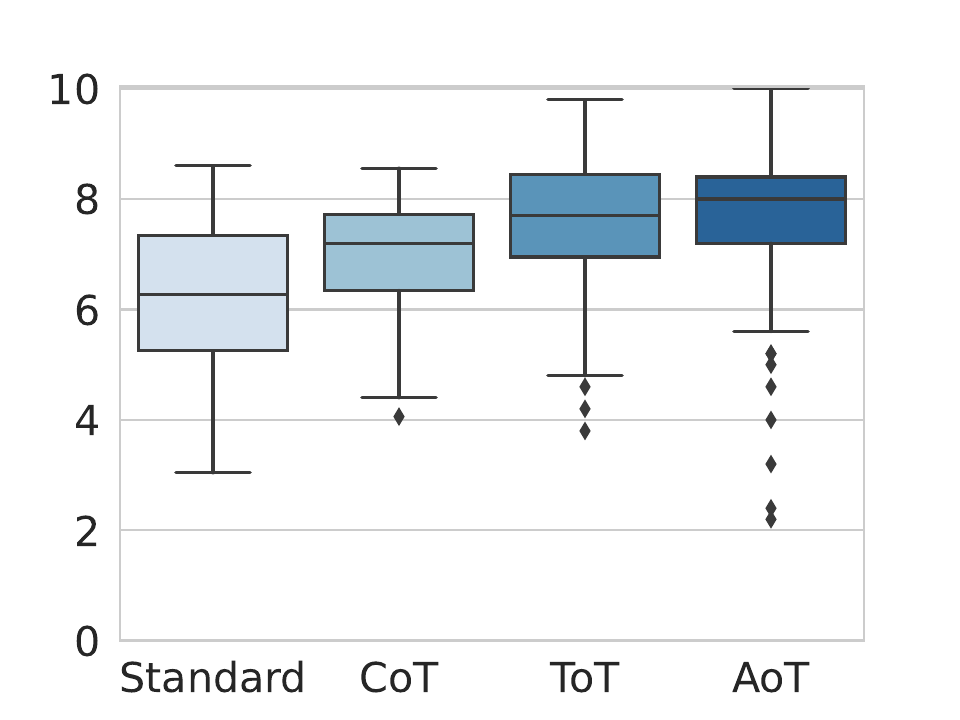}
\caption{Comparison of the standard prompting, CoT, ToT and AoT on the creative writing task.}
\label{fig: text_performance}
\end{figure}

\renewcommand{\arraystretch}{1.2} 
\begin{table}[h]
\centering
\begin{tabular}{lll}
\hline
\textbf{Method}    & \textbf{Score} & \textbf{Avg. Queries} \\ \hline
Standard Prompting & $6.19$           & $1$                     \\
CoT                & $6.93$           & $1$                     \\
ToT                & $7.56$           & $20$                    \\
AoT                & $\mathbf{7.58}$           & $1$                     \\ \hline
\end{tabular}
\caption{Performance of the methods determined by GPT-4.}
\label{table: creative_writing_results}
\end{table}

\section{CoT vs. Single Iteration AoT in the Game of 24}
To demonstrate that the tree search mechanism is fundamentally distinct from the CoT prompting, even in scenarios where AoT's in-context examples include only a single initial operation in the game of 24, we draw a comparison between AoT (Short) and CoT. In this setup, AoT (Short) determines the first operation and subsequently conducts a tree search on the remaining three numbers. Interestingly, AoT (Short) achieves a success rate of $48\%$, while CoT lags significantly, securing only $4\%$. These results underscore the notion that even a rudimentary search mechanism can lead to significant performance enhancements.

\section{Detailed Analysis on the Effect of the Length of the Prompts}
In this section, we delve deeper into Fig. \ref{fig: short_vs_normal_vs_long} by presenting histograms for the successful, unsuccessful, and total games of `24', considering the number of initial steps in methods AoT (Short), AoT, and AoT (Long). These are displayed in Figs. \ref{fig: text_performance_successful}-\ref{fig: text_performance_all}.

From these figures, it becomes evident that the length of the prompts, measured by the number of initial steps included in in-context examples, correlates with the length of their solutions to test examples. This trend is consistent across all three cases, suggesting that AoT's strategy in determining the number of initial steps is influenced by its in-context examples.

Interestingly, when AoT is provided a well-balanced set of initial steps that emphasize the most promising operations, it excels in solving the majority of games in earlier iterations. This indicates AoT's capacity to prioritize swift problem-solving without sacrificing performance. This tendency is also observed in AoT (Long), albeit with a somewhat reduced success rate, as illustrated in Fig. \ref{fig: text_performance_successful}.

\begin{figure}[h!]
\centering
\includegraphics[width=0.8\textwidth]{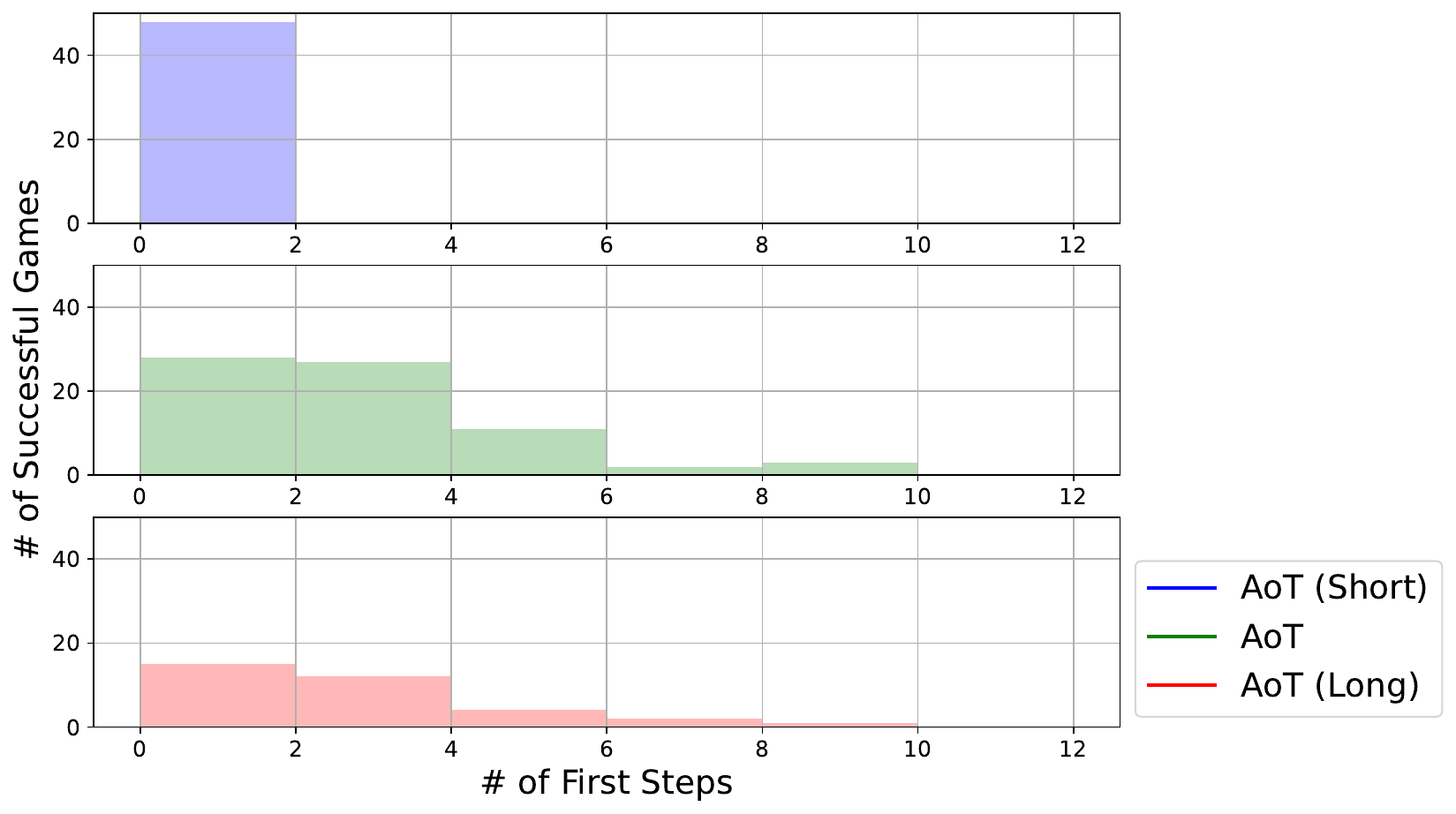}
\caption{Histogram of the number of successful games with respect to the number of first steps for AoT (Short), AoT and AoT (Long).}
\label{fig: text_performance_successful}
\end{figure}

\begin{figure}[h!]
\centering
\includegraphics[width=0.8\textwidth]{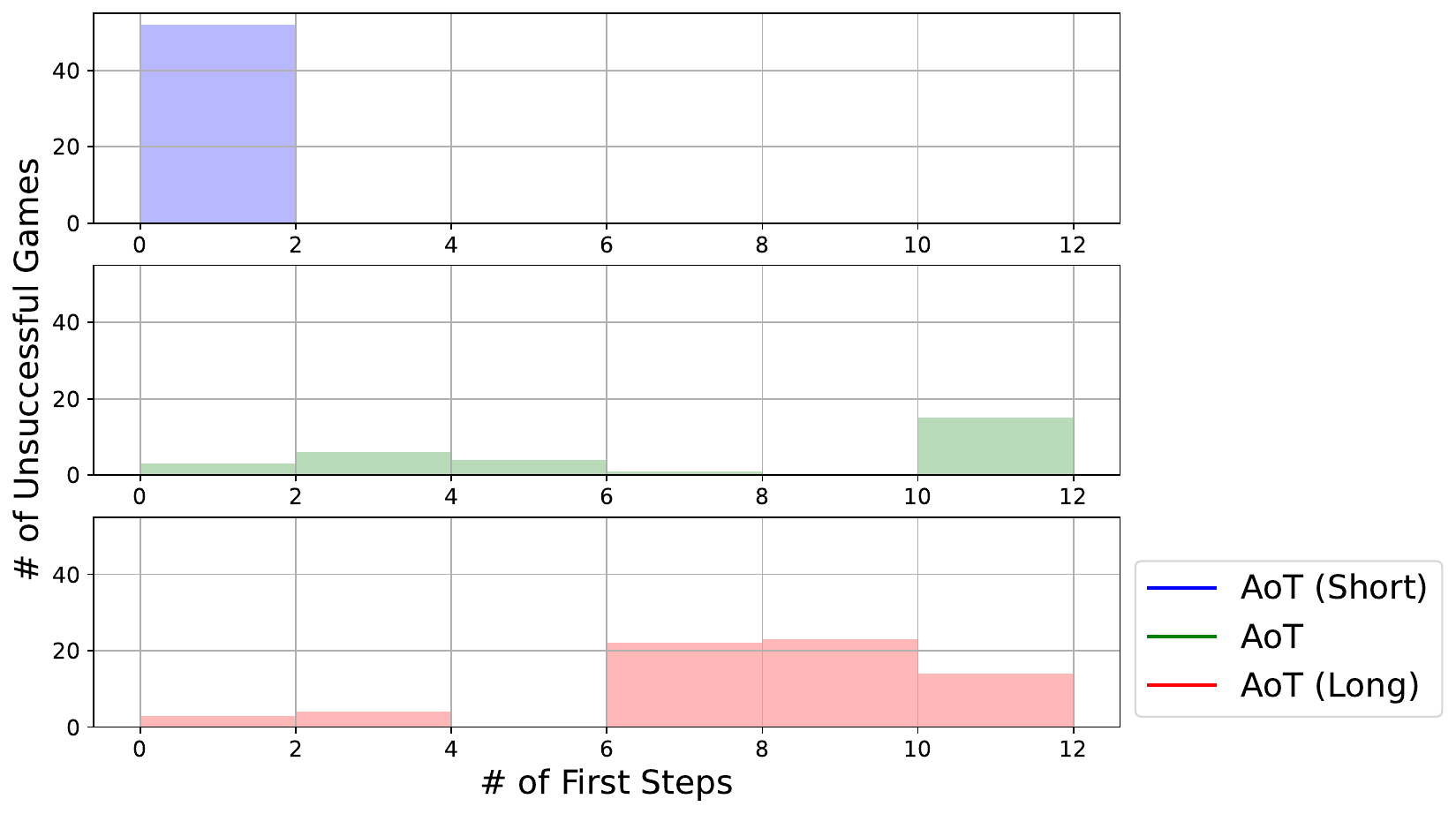}
\caption{Histogram of the number of unsuccessful games with respect to the number of first steps for AoT (Short), AoT and AoT (Long).}
\label{fig: text_performance_unsuccessful}
\end{figure}

\begin{figure}[h!]
\centering
\includegraphics[width=0.8\textwidth]{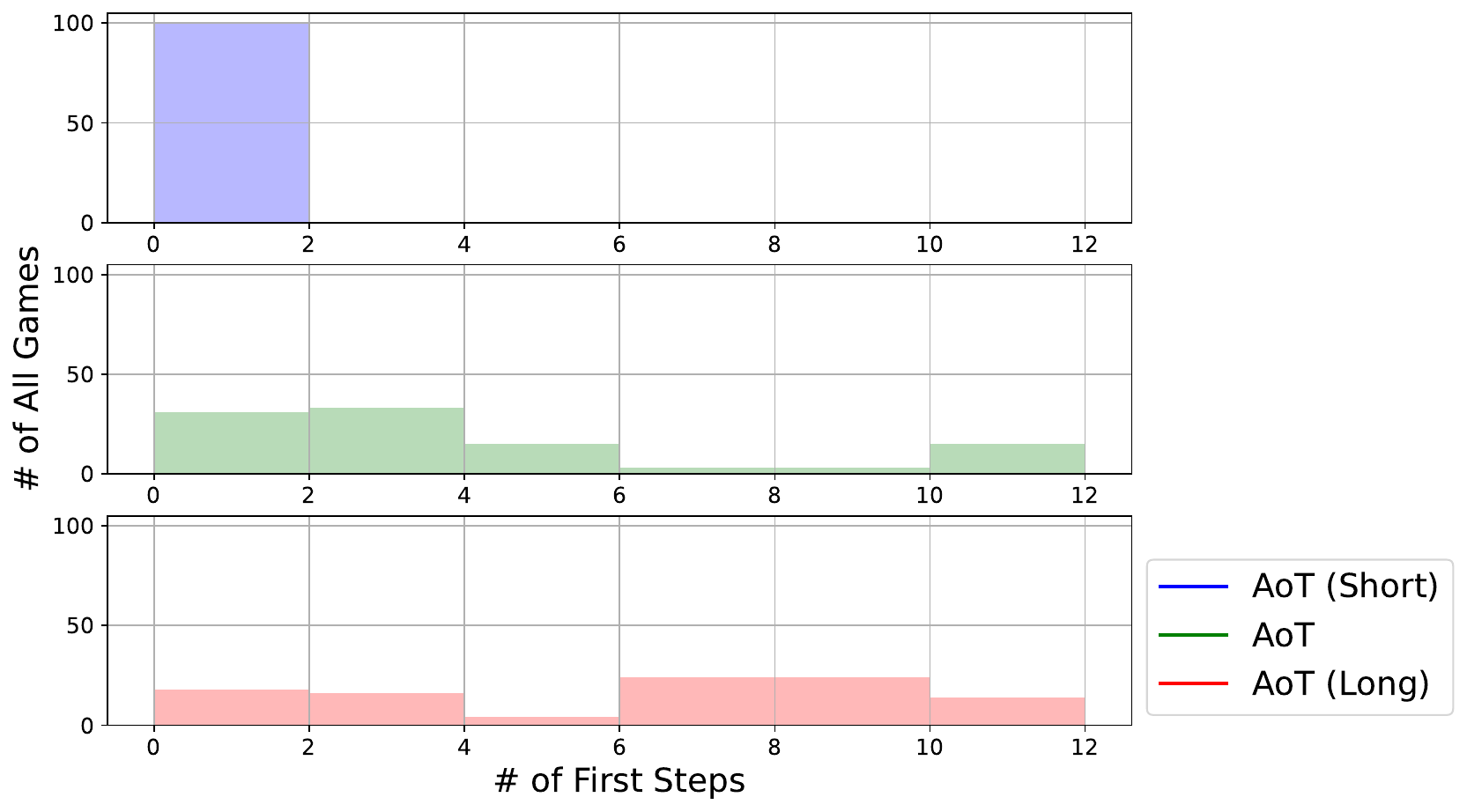}
\caption{Histogram of the number of all games with respect to the number of first steps for AoT (Short), AoT and AoT (Long).}
\label{fig: text_performance_all}
\end{figure}

\section{Proof of Corollary \ref{corollary: main}}
\begin{corollary}
    Consider \textsc{TIME}$(a^n)$ as the class of languages L for which a Turing machine exists that operates within a time complexity of $O(a^n)$ for some $a\geq 1$. For a transformer generating $O(a^n)$ intermediate tokens with AoT, we have
    \begin{equation}
        \textsc{TIME}(a^n) \subseteq \textsc{AoT}(n),
    \end{equation}
    where $\textsc{AoT}(n)$ refers to the decoding steps by AoT when the input has $n$ tokens.
\end{corollary}
\begin{proof}
    Since with AoT prompting, we can backtrack and continue from other nodes, the number of intermediate tokens scale exponentially with the depth of the problem over the number of possible actions in each leaf node. Then, directly by Theorem 2. in \citet{merrill2023expresssive}, we have the stated result implying the capability of solving NP problems with AoT prompting.
\end{proof}

\section{Prompts}

\subsection{Game of 24}
Below, we represent the specific prompts employed for the various methods detailed in the experiments section. It's important to note that the terms ``System'',``User'', and ``Assistant'' are utilized to denote the \textit{roles} within the OpenAI API when operating in chat completion mode. The line breaks serve to show the transitions between the user and assistant interactions within the API.

\paragraph{AoT (DFS)}
\leavevmode\begin{lstlisting}[numbers=none]
System:
        Use numbers and basic arithmetic operations (+ - * /) to obtain 24. When considering the next steps, do not choose operations that will result in a negative or fractional number. In order to help with the calculations, the numbers in the parenthesis represent the numbers that are left after the operations and they are in descending order.
        
        Another thing we do is when there are only two numbers left in the parenthesis, we check whether we can arrive at 24 only by using basic arithmetic operations (+ - * /). Some examples regarding this idea:
        (21 2) no
        since 21 + 2 = 23, 21 - 2 = 19, 21 * 2 = 42, 21 / 2 = 10.5, none of which is equal to 24.
        (30 6) 30 - 6 = 24 yes
        (8 3) 8 * 3 = 24 yes
        (12 8) no
        (48 2) 48 / 2 = 24 yes
        
        Most importantly, do not give up, all the numbers that will be given has indeed a solution.
~~~~~~~~~~~~~~~~~~~~~~~~~~~~~~~~~~~~~~~~~~~~~~~~~
User:
        14 8 8 2
Assistant:
        Trying a promising first operation:
        1. 8 + 2: (14, 10, 8)
        - 14 + 10: (24, 8) 32, 16, 192, 3
        - 14 - 10: (8, 4) 12, 8, 32, 2
        - 14 * 10: (140, 8) 148, 132, 1120, fractional
        - 14 / 10: (8, 1.4) fractional
        - 14 + 8: (22, 10) 32, 12, 220, fractional
        - 14 - 8: (10, 6) 16, 4, 60, fractional
        - 14 * 8: (112, 10) 122, 102, 1120, fractional
        - 14 / 8: fractional
        - 10 + 8: (18, 14) 32, 4, 252, fractional
        - 10 - 8: (14, 2) 16, 12, 28, 7
        - 10 * 8: (80, 14) 94, 66, big, fractional
        - 10 / 8: fractional
        
        Trying another promising first operation:
        2. 8 / 2: (14, 8, 4)
        - 14 + 8: (22, 4) 26, 18, 88, fractional
        - 14 - 8: (6, 4) 10, 2, 24 = 6 * 4 -> found it!
        Backtracking the solution:
        Step 1:
        8 / 2 = 4
        Step 2:
        14 - 8 = 6
        Step 3:
        6 * 4 = 24
        Considering these steps: 24 = 6 * 4 = (14 - 8) * 4 = (14 - 8) * (8 / 2) = 24.
        answer: (14 - 8) * (8 / 2) = 24.
~~~~~~~~~~~~~~~~~~~~~~~~~~~~~~~~~~~~~~~~~~~~~~~~~
User:
        9 5 5 5
Assistant:
        Trying a promising first operation:
        1. 9 - 5: (5, 5, 4)
        - 5 + 5: (10, 4) 14, 6, 40, fractional
        - 5 - 5: (4, 0) 4, 4, 0, undefined
        - 5 * 5: (25, 4) 29, 21, 100, fractional
        - 5 / 5: (4, 1) 5, 3, 4, 4
        - 5 + 4: (9, 5) 14, 4, 45, fractional
        - 5 - 4: (5, 1) 6, 4, 5, 0.2
        - 5 * 4: (20, 5) 25, 15, 100, fractional
        - 5 / 4: fractional
        
        Trying another promising first operation:
        2. 5 + 5: (10, 9, 5)
        - 10 + 9: (19, 5) 24 = 19 + 5 -> found it!
        Backtracking the solution:
        Step 1:
        5 + 5 = 10
        Step 2:
        10 + 9 = 19
        Step 3:
        19 + 5 = 24
        Considering these steps: 24 = 19 + 5 = (10 + 9) + 5 = ((5 + 5) + 9) + 5 = 24.
        answer: ((5 + 5) + 9) + 5 = 24.
~~~~~~~~~~~~~~~~~~~~~~~~~~~~~~~~~~~~~~~~~~~~~~~~~
User:
        8 6 4 4
Assistant:
        Trying a promising first operation:
        1. 8 - 6: (4, 4, 2)
        - 4 + 4: (8, 2) 10, 6, 16, 4
        - 4 - 4: (2, 0) 2, 2, 0, undefined
        - 4 * 4: (16, 2) 18, 14, 32, 8
        - 4 / 4: (2, 1) 3, 1, 2, 2
        - 4 + 2: (6, 4) 10, 2, 24 = 6 * 4 -> found it!
        Backtracking the solution:
        Step 1:
        8 - 6 = 2
        Step 2:
        4 + 2 = 6
        Step 3:
        6 * 4 = 24
        Considering these steps: 24 = 6 * 4 = (4 + 2) * 4 = (4 + (8 - 6)) * 4 = 24.
        answer: (4 + (8 - 6)) * 4 = 24.
~~~~~~~~~~~~~~~~~~~~~~~~~~~~~~~~~~~~~~~~~~~~~~~~~
User:
        13 10 9 4
Assistant:
        Trying a promising first operation:
        1. 13 - 4: (10, 9, 9)
        - 10 + 9: (19, 9) 28, 10, 171, fractional
        - 10 - 9: (9, 1) 10, 8, 9, 9
        - 10 * 9: (90, 9) 99, 81, 810, fractional
        - 10 / 9: fractional
        - 9 + 9: (18, 10) 28, 8, 180, fractional
        - 9 - 9: (10, 0) 10, 10, 0, undefined
        - 9 * 9: (81, 10) 91, 71, 810, fractional
        - 9 / 9: (10, 1) 11, 9, 10, 10
        
        Trying another promising first operation:
        2. 10 - 4: (13, 9, 6)
        - 13 + 9: (22, 6) 28, 16, 132, fractional
        - 13 - 9: (6, 4) 10, 2, 24 = 6 * 4 -> found it!
        Backtracking the solution:
        Step 1:
        10 - 4 = 6
        Step 2:
        13 - 9 = 4
        Step 3:
        6 * 4 = 24
        Considering these steps: 24 = 6 * 4 = 6 * (13 - 9) = (10 - 4) * (13 - 9) = 24.
        answer: (10 - 4) * (13 - 9) = 24.
~~~~~~~~~~~~~~~~~~~~~~~~~~~~~~~~~~~~~~~~~~~~~~~~~
User:
        8 8 5 4
Assistant:
        Trying a promising first operation:
        1. 8 - 4: (8, 5, 4)
        - 8 + 5: (13, 4) 17, 9, 52, fractional
        - 8 - 5: (4, 3) 7, 1, 12, fractional
        - 8 * 5: (40, 4) 44, 36, 160, fractional
        - 8 / 5: fractional
        - 8 + 4: (12, 5) 17, 7, 60, fractional
        - 8 - 4: (5, 4) 9, 1, 20, fractional
        - 8 * 4: (32, 5) 37, 27, 160, fractional
        - 8 / 4: (5, 2) 7, 3, 10, 2.5
        - 5 + 4: (9, 8) 17, 1, 72, fractional
        - 5 - 4: (8, 1) 9, 7, 8, 8
        - 5 * 4: (20, 8) 28, 12, 160, fractional
        - 5 / 4: fractional
        
        Trying another promising first operation:
        2. 8 / 4: (8, 5, 2)
        - 8 + 5: (13, 2) 15, 11, 26, fractional
        - 8 - 5: (3, 2) 5, 1, 6, 1.5
        - 8 * 5: (40, 2) 42, 38, 80, 20
        - 8 / 5: fractional
        - 8 + 2: (10, 5) 15, 5, 50, 2
        - 8 - 2: (6, 5) 11, 1, 30, fractional
        - 8 * 2: (16, 5) 21, 11, 80, fractional
        - 8 / 2: (5, 4) 9, 1, 20, fractional
        - 5 + 2: (7, 8) 15, 1, 56, fractional
        - 5 - 2: (8, 3) 11, 5, 24 = 8 * 3 -> found it!
        Backtracking the solution:
        Step 1:
        8 / 4 = 2
        Step 2:
        5 - 2 = 3
        Step 3:
        8 * 3 = 24
        Considering these steps: 24 = 8 * 3 = 8 * (5 - 2) = 8 * (5 - (8 / 4)) = 24.
        answer: 8 * (5 - (8 / 4)) = 24.
~~~~~~~~~~~~~~~~~~~~~~~~~~~~~~~~~~~~~~~~~~~~~~~~~
User:
        11 11 1 1
Assistant:
        Trying a promising first operation:
        1. 11 - 1: (11, 10, 1)
        - 11 + 10: (21, 1) 22, 20, 21, 21
        - 11 - 10: (1, 1) 2, 0, 1, 1
        - 11 * 10: (110, 1) 111, 109, 110, 110
        - 11 / 10: fractional
        - 11 + 1: (12, 10) 22, 2, 120, 1.2
        - 11 - 1: (10, 10) 20, 0, 100, 1
        - 11 * 1: (11, 10) 21, 1, 110, 1.1
        - 11 / 1: (11, 10) 21, 1, 110, fractional
        - 10 + 1: (11, 11) 22, 0, 121, 1
        - 10 - 1: (11, 9) 20, 2, 99, 1.1
        - 10 * 1: (11, 10) 21, 1, 110, 1.1
        - 10 / 1: (11, 10) 21, 1, 110, 1.1
        
        Trying another promising first operation:
        2. 11 + 1: (12, 11, 1)
        - 12 + 11: (23, 1) 24 = 23 + 1 -> found it!
        Backtracking the solution:
        Step 1:
        11 + 1 = 12
        Step 2:
        12 + 11 = 23
        Step 3:
        23 + 1
        Considering these steps: 24 = 23 + 1 = (12 + 11) + 1 = ((11 + 1) + 11) + 1 = 24.
        answer: ((11 + 1) + 11) + 1 = 24.
~~~~~~~~~~~~~~~~~~~~~~~~~~~~~~~~~~~~~~~~~~~~~~~~~
User:
        11 7 4 1
Assistant:
        Trying a promising first operation:
        1. 7 * 4: (28, 11, 1)
        - 28 + 11: (39, 1) 40, 38, 39, 39
        - 28 - 11: (17, 1) 18, 16, 17, 17
        - 28 * 11: (308, 1) 309, 307, 308, 308
        - 28 / 11: fractional
        - 28 + 1: (29, 11) 40, 18, 319, fractional
        - 28 - 1: (27, 11) 38, 16, 297, fractional
        - 28 * 1: (28, 11) 39, 17, 308, fractional
        - 28 / 1: (28, 11) 39, 17, 308, fractional
        - 11 + 1: (29, 28) 57, 1, 812, fractional
        - 11 - 1: (28, 10) 38, 18, 280, fractional
        - 11 * 1: (28, 11) 39, 17, 308, fractional
        - 11 / 1: (28, 11) 39, 17, 308, fractional
        
        Trying another promising first operation:
        2. 7 + 1: (11 8 4)
        - 11 + 8: (19, 4) 23, 15, 76, fractional
        - 11 - 8: (4, 3) 7, 1, 12, fractional
        - 11 * 8: (88, 4) 92, 84, 352, fractional
        - 11 / 8: fractional
        - 11 + 4: (15, 8) 23, 7, 120, fractional
        - 11 - 4: (7, 8) 15, -1, 56, fractional
        - 11 * 4: (44, 8) 52, 36, 352, fractional
        - 11 / 4: fractional
        - 8 + 4: (12, 11) 23, -1, 132, fractional
        - 8 - 4: (11, 4) 15, 7, 44, fractional
        - 8 * 4: (32, 11) 43, 21, 352, fractional
        - 8 / 4: (11, 2) 13, 9, 22, fractional
        
        Trying another promising first operation:
        3. 4 + 1: (11 7 5)
        - 11 + 7: (18, 5) 23, 13, 90, fractional
        - 11 - 7: (5, 4) 9, 1, 20, fractional
        - 11 * 7: (77, 5) 82, 72, 385, fractional
        - 11 / 7: fractional
        - 11 + 5: (16, 7) 23, 9, 112, fractional
        - 11 - 5: (7, 6) 13, 1, 42, fractional
        - 11 * 5: (55, 7) 62, 48, 385, fractional
        - 11 / 5: fractional
        - 7 + 5: (12, 11) 23, 1, 132, fractional
        - 7 - 5: (11, 2) 13, 9, 22, fractional
        - 7 * 5: (35, 11) 46, 24 = 35 - 11 -> found it!
        Step 1:
        4 + 1 = 5
        Step 2:
        7 * 5 = 35
        Step 3:
        35 - 11 = 24
        Considering these steps: Backtracking the solution:
        24 = 35 - 11 = (7 * 5) - 11 = (7 * (4 + 1)) - 11 = 24.
        answer: (7 * (4 + 1)) - 11 = 24.
~~~~~~~~~~~~~~~~~~~~~~~~~~~~~~~~~~~~~~~~~~~~~~~~~
User:
        11 5 4 3
Assistant:
        Trying a promising first operation:
        1. 5 * 4: (20, 11, 3)
        - 20 + 11: (31, 3) 34, 28, 93, fractional
        - 20 - 11: (9, 3) 12, 6, 27, 3
        - 20 * 11: (220, 3) 223, 217, 660, fractional
        - 20 / 11: fractional
        - 20 + 3: (23, 11) 34, 12, 253, fractional
        - 20 - 3: (17, 11) 28, 6, 187, fractional
        - 20 * 3: (60, 11) 71, 49, 660, fractional
        - 20 / 3: fractional
        - 11 + 3: (14, 20) 34, -6, 280, fractional
        - 11 - 3: (8, 20) 28, -12, 160, fractional
        - 11 * 3: (33, 20) 53, 13, 660, fractional
        - 11 / 3: fractional
        
        Trying another promising first operation:
        2. 11 - 3: (8, 5, 4)
        - 8 + 5: (13, 4) 17, 9, 52, fractional
        - 8 - 5: (4, 3) 7, 1, 12, fractional
        - 8 * 5: (40, 4) 44, 36, 160, fractional
        - 8 / 5: fractional
        - 8 + 4: (12, 5) 17, 7, 60, fractional
        - 8 - 4: (5, 4) 9, 1, 20, fractional
        - 8 * 4: (32, 5) 37, 27, 160, fractional
        - 8 / 4: (5, 2) 7, 3, 10, 2.5
        - 5 + 4: (9, 8) 17, 1, 72, fractional
        - 5 - 4: (8, 1) 9, 7, 8, 8
        - 5 * 4: (20, 8) 28, 12, 160, fractional
        - 5 / 4: fractional
        
        Trying another promising first operation:
        3. 11 * 3: (33, 5, 4)
        - 33 + 5: (38, 4) 42, 34, 152, fractional
        - 33 - 5: (28, 4) 32, 24 = 28 - 4 -> found it!
        Backtracking the solution:
        Step 1:
        11 * 3 = 33
        Step 2:
        33 - 5 = 28
        Step 3:
        28 - 4 = 24
        Considering these steps: 24 = 28 - 4 = (33 - 5) - 4 = ((11 * 3) - 5) - 4 = 24.
        answer: ((11 * 3) - 5) - 4 = 24.
~~~~~~~~~~~~~~~~~~~~~~~~~~~~~~~~~~~~~~~~~~~~~~~~~
User:
        13 12 5 2
Assistant:
        Trying a promising first operation:
        1. 12 + 2: (14, 13, 5)
        - 14 + 13: (27, 5) 32, 22, 135, fractional
        - 14 - 13: (5, 1) 6, 4, 5, 5
        - 14 * 13: (182, 5) 187, 177, 910, fractional
        - 14 / 13: fractional
        - 14 + 5: (19, 13) 32, 6, 247, fractional
        - 14 - 5: (13, 9) 22, 4, 117, fractional
        - 14 * 5: (70, 13) 83, 57, 910, fractional
        - 14 / 5: fractional
        - 13 + 5: (18, 14) 32, 4, 252, fractional
        - 13 - 5: (14, 8) 22, 6, 112, fractional
        - 13 * 5: (65, 14) 79, 51, 910, fractional
        - 13 / 5: fractional
        
        Trying another promising first operation:
        2. 13 + 5: (18, 12, 2)
        - 18 + 12: (30, 2) 32, 28, 60, fractional
        - 18 - 12: (6, 2) 8, 4, 12, 3
        - 18 * 12: (216, 2) 218, 214, 432, fractional
        - 18 / 12: fractional
        - 18 + 2: (20, 12) 32, 8, 240, fractional
        - 18 - 2: (16, 12) 28, 4, 192, fractional
        - 18 * 2: (36, 12) 48, 24 = 36 - 12 -> found it!
        Backtracking the solution:
        Step 1:
        13 + 5 = 18
        Step 2:
        18 * 2 = 36
        Step 3:
        36 - 12 = 24
        Considering these steps: 24 = 36 - 12 = (18 * 2) - 12 = ((13 + 5) * 2) - 12 = 24.
        answer: ((13 + 5) * 2) - 12 = 24.
~~~~~~~~~~~~~~~~~~~~~~~~~~~~~~~~~~~~~~~~~~~~~~~~~
User:
        9 8 2 1
Assistant:
        Trying a promising first operation:
        1. 9 - 1: (8, 8, 2)
        - 8 + 8: (16, 2) 18, 14, 32, 8
        - 8 - 8: (2, 0) 2, 2, 0, undefined
        - 8 * 8: (64, 2) 66, 62, 128, 32
        - 8 / 8: (2, 1) 3, 1, 2, 2
        - 8 + 2: (10, 8) 18, 2, 80, 4
        - 8 - 2: (6, 8) 14, -2, 48, fractional
        - 8 * 2: (16, 8) 24 = 16 + 8 -> found it!
        Backtracking the solution:
        Step 1:
        9 - 1 = 8
        Step 2:
        8 * 2 = 16
        Step 3:
        16 + 8 = 24
        Considering these steps: 24 = 16 + 8 = (8 * 2) + 8 = ((9 - 1) * 2) + 8 = 24.
        answer: ((9 - 1) * 2) + 8 = 24.
~~~~~~~~~~~~~~~~~~~~~~~~~~~~~~~~~~~~~~~~~~~~~~~~~
\end{lstlisting}

\subsubsection{AoT (Long)}
\leavevmode\begin{lstlisting}[numbers=none]
System:
        Use numbers and basic arithmetic operations (+ - * /) to obtain 24. When considering the next steps, do not choose operations that will result in a negative or fractional number. In order to help with the calculations, the numbers in the parenthesis represent the numbers that are left after the operations and they are in descending order.
        
        Another thing we do is when there are only two numbers left in the parenthesis, we check whether we can arrive at 24 only by using basic arithmetic operations (+ - * /). Some examples regarding this idea:
        (21 2) no
        since 21 + 2 = 23, 21 - 2 = 19, 21 * 2 = 42, 21 / 2 = 10.5, none of which is equal to 24.
        (30 6) 30 - 6 = 24 yes
        (8 3) 8 * 3 = 24 yes
        (12 8) no
        (48 2) 48 / 2 = 24 yes

        Most importantly, do not give up, all the numbers that will be given has indeed a solution.
User:
        14 8 8 2
Assistant:
        Trying a promising first operation:
        1. 8 + 2: (14, 10, 8)
        - 14 + 10: (24, 8) 32, 16, 192, 3
        - 14 - 10: (8, 4) 12, 8, 32, 2
        - 14 * 10: (140, 8) 148, 132, 1120, fractional
        - 14 / 10: (8, 1.4) fractional
        - 14 + 8: (22, 10) 32, 12, 220, fractional
        - 14 - 8: (10, 6) 16, 4, 60, fractional
        - 14 * 8: (112, 10) 122, 102, 1120, fractional
        - 14 / 8: fractional
        - 10 + 8: (18, 14) 32, 4, 252, fractional
        - 10 - 8: (14, 2) 16, 12, 28, 7
        - 10 * 8: (80, 14) 94, 66, big, fractional
        - 10 / 8: fractional
        
        Trying another promising first operation:
        2. 14 + 8: (22, 8, 2)
        - 22 + 8: (30, 2) 32, 28, 60, 15
        - 22 - 8: (14, 2) 16, 12, 28, 7
        - 22 * 8: (176, 2) 178, 174, 88
        - 22 / 8: (2.75, 2) fractional
        - 22 + 2: (24, 8) 32, 16, 192, 3
        - 22 - 2: (20, 8) 28, 12, 160, fractional
        - 22 * 2: (44, 8) 52, 36, 352, fractional
        - 22 / 2: (11, 8) 19, 3, 88, fractional
        - 8 + 2: (22, 10) 32, 12, 220, fractional
        - 8 - 2: (22, 6) 28, 16, 132, fractional
        - 8 * 2: (22, 16) 38, 6, 352, fractional
        - 8 / 2: (22, 4) 26, 18, 88, fractional
        
        Trying another promising first operation:
        3. 14 + 2: (16, 8, 8)
        - 16 + 8: (24, 8) 32, 16, 192, 3
        - 16 - 8: (8, 8) 16, 0, 64, 1
        - 16 * 8: (128, 8) 136, 120, 1024, 16
        - 16 / 8: (8, 2) 10, 6, 16, 4
        - 8 + 8: (16, 16 32, 0, 256, 1
        - 8 - 8: (16, 0) 16, 16, 0, undefined
        - 8 * 8: (64, 16) 80, 48, 1024, 4
        - 8 / 8: (16, 1) 17, 15, 16, 16
        
        Trying another promising first operation:
        4. 8 - 2: (14, 8, 6)
        - 14 + 8: (22, 14) 36, 8, 308, fractional
        - 14 - 8: (6, 6) 12, 0, 36, 1
        - 14 * 8: (112, 6) 118, 106, 672, fractional
        - 14 / 8: (6, 1.75) fractional
        - 14 + 6: (20, 8) 22, 12, 160, fractional
        - 14 - 6: (8, 8) 16, 0, 64, 1
        - 14 * 6: (84, 8) 92, 76, 672, fractional
        - 14 / 6: (8, 2.3) fractional
        - 8 + 6: (14, 14) 28, 0, 196, 1
        - 8 - 6: (14, 2) 16, 12, 28, 7
        - 8 * 6: (48, 14) 62, 34, 672, fractional
        - 8 / 6: (14, 1.3) fractional
        
        Trying another promising first operation:
        5. 8 * 2: (16, 14, 8)
        - 16 + 14: (30, 8) 38, 22, 240, fractional
        - 16 - 14: (8, 2) 10, 6, 16, 4
        - 16 * 14: (224, 8) 232, 216, 1792, 28
        - 16 / 14: (8, 1.1) fractional
        - 16 + 8: (24, 14) 38, 10, 336, fractional
        - 16 - 8: (14, 8) 22, 6, 112, fractional
        - 16 * 8: (128, 14) 142, 112, 1792, fractional
        - 16 / 8: (14, 2) 16, 12, 28, 7
        - 14 + 8: (22, 16) 38, 6, 352, fractional
        - 14 - 8: (16, 6) 22, 10, 96, fractional
        - 14 * 8: (112, 16) 128, 96, 1792, 7
        - 14 / 8: (16, 1.7) fractional
        
        Trying another promising first operation:
        6. 14 * 2: (28, 8, 8)
        - 28 + 8: (36, 8) 44, 28, 288, fractional
        - 28 - 8: (20, 8) 28, 12, 160, fractional
        - 28 * 8: (224, 8) 232, 216, 1792, 28
        - 28 / 8: (8, 3.5) fractional, fractional, 28, fractional
        - 8 + 8: (16, 16 32, 0, 256, 1
        - 8 - 8: (16, 0) 16, 16, 0, undefined
        - 8 * 8: (64, 16) 80, 48, 1024, 4
        - 8 / 8: (16, 1) 17, 15, 16, 16
        
        Trying another promising first operation:
        7. 8 / 2: (14, 8, 4)
        - 14 + 8: (22, 4) 26, 18, 88, fractional
        - 14 - 8: (6, 4) 10, 2, 24 = 6 * 4 -> found it!
        Backtracking the solution:
        Step 1:
        8 / 2 = 4
        Step 2:
        14 - 8 = 6
        Step 3:
        6 * 4 = 24
        Considering these steps: 24 = 6 * 4 = (14 - 8) * 4 = (14 - 8) * (8 / 2) = 24.
        answer: (14 - 8) * (8 / 2) = 24.
~~~~~~~~~~~~~~~~~~~~~~~~~~~~~~~~~~~~~~~~~~~~~~~~~
User:
        9 5 5 5
Assistant:
        Trying a promising first operation:
        1. 9 - 5: (5, 5, 4)
        - 5 + 5: (10, 4) 14, 6, 40, fractional
        - 5 - 5: (4, 0) 4, 4, 0, undefined
        - 5 * 5: (25, 4) 29, 21, 100, fractional
        - 5 / 5: (4, 1) 5, 3, 4, 4
        - 5 + 4: (9, 5) 14, 4, 45, fractional
        - 5 - 4: (5, 1) 6, 4, 5, 0.2
        - 5 * 4: (20, 5) 25, 15, 100, fractional
        - 5 / 4: fractional
        
        Trying another promising first operation:
        2. 5 * 5: (25, 9, 5)
        - 25 + 9: (34, 5) 39, 29, 170, fractional
        - 25 - 9: (16, 5) 21, 11, 80, fractional
        - 25 * 9: (225, 5) 230, 220, 1125, 45
        - 25 / 9: (5, 2.7) fractional
        - 25 + 5: (30, 9) 39, 21, 270, fractional
        - 25 - 5: (20, 9) 29, 11, 180, fractional
        - 25 * 5: (75, 9) 84, 66, 675, fractional
        - 25 / 5: (9, 5) 14, 4, 45, fractional
        - 9 + 5: (25, 14) 39, 11, 350, fractional
        - 9 - 5: (25, 4) 29, 21, 100, fractional
        - 9 * 5: (45, 25) 70, 20, 1125, fractional
        - 9 / 5: (25, 1.8) fractional, fractional, 45, fractional
        
        Trying another promising first operation:
        3. 5 - 5: (9, 5, 0)
        - 9 + 5: (25, 14) 39, 11, 350, fractional
        - 9 - 5: (25, 4) 29, 21, 100, fractional
        - 9 * 5: (45, 25) 70, 20, 1125, fractional
        - 9 / 5: (25, 1.8) fractional, fractional, 45, fractional
        - 9 + 0: (9, 5) 14, 4, 45, fractional
        - 9 - 0: (9, 5) 14, 4, 45, fractional
        - 9 * 0: (5, 0) 5, 5, 0, undefined
        - 9 / 0: undefined
        - 5 + 0: (9, 5) 14, 4, 45, fractional
        - 5 - 0: (9, 5) 14, 4, 45, fractional
        - 5 * 0: (9, 0) 9, 9, 0, undefined
        - 5 / 0: undefined
        
        Trying another promising first operation:
        4. 5 / 5: (9, 5, 1)
        - 9 + 5: (25, 14) 39, 11, 350, fractional
        - 9 - 5: (25, 4) 29, 21, 100, fractional
        - 9 * 5: (45, 25) 70, 20, 1125, fractional
        - 9 / 5: (25, 1.8) fractional, fractional, 45, fractional
        - 9 + 1: (10, 5) 15, 5, 50, 2
        - 9 - 1: (8, 5) 13, 3, 40, fractional
        - 9 * 1: (9, 5) 14, 4, 45, fractional
        - 9 / 1: (9, 5) 14, 4, 45, fractional
        - 5 + 1: (9, 6) 15, 3, 54, fractional
        - 5 - 1: (9, 4) 13, 5, 36, fractional
        - 5 * 1: (9, 5) 14, 4, 45, fractional
        - 5 / 1: (9, 5) 14, 4, 45, fractional
        
        Trying another promising first operation:
        5. 9 * 5: (45, 5, 5)
        - 45 + 5: (50, 5) 55, 45, 250, 10
        - 45 - 5: (40, 5) 45, 35, 200, 8
        - 45 * 5: (225, 5) 230, 220, 1125, 45
        - 45 / 5: (9, 5) 14, 4, 45, fractional
        - 5 + 5: (10, 4) 14, 6, 40, fractional
        - 5 - 5: (4, 0) 4, 4, 0, undefined
        - 5 * 5: (25, 4) 29, 21, 100, fractional
        - 5 / 5: (4, 1) 5, 3, 4, 4
        
        Trying another promising first operation:
        6. 5 + 5: (10, 9, 5)
        - 10 + 9: (19, 5) 24 = 19 + 5 -> found it!
        Backtracking the solution:
        Step 1:
        5 + 5 = 10
        Step 2:
        10 + 9 = 19
        Step 3:
        19 + 5 = 24
        Considering these steps: 24 = 19 + 5 = (10 + 9) + 5 = ((5 + 5) + 9) + 5 = 24.
        answer: ((5 + 5) + 9) + 5 = 24.
~~~~~~~~~~~~~~~~~~~~~~~~~~~~~~~~~~~~~~~~~~~~~~~~~
User:
        8 6 4 4
Assistant:
        Trying a promising first operation:
        1. 8 * 6: (48, 4, 4)
        - 48 + 4: (52, 4) 56, 48, 208, 13
        - 48 - 4: (44, 4) 48, 40, 176, 11
        - 48 * 4: (192, 4) 196, 188, 768, 48
        - 48 / 4: (12, 4) 16, 8, 48, 3
        - 4 + 4: (48, 8) 56, 40, 384, 6
        - 4 - 4: (48, 0) 48, 48, 0, undefined
        - 4 * 4: (48, 16) 64, 32, 768, 3
        - 4 / 4: (48, 1) 49, 47, 48, 48
        
        Trying another promising first operation:
        2. 4 - 4: (8, 6, 0)
        - 8 + 6: (14, 0) 14, 14, 0, undefined
        - 8 - 6: (2, 0) 2, 2, 0, undefined
        - 8 * 6: (48, 0) 48, 48, 0, undefined
        - 8 / 6: (1.3, 0) fractional
        - 8 + 0: (8, 6) 14, 2, 48, fractional
        - 8 - 0: (8, 6) 14, 2, 48, fractional
        - 8 * 0: (6, 0) 6, 6, 0, undefined
        - 8 / 0: undefined
        - 6 + 0: (8, 6) 14, 2, 48, fractional
        - 6 - 0: (8, 6) 14, 2, 48, fractional
        - 6 * 0: (8, 0) 8, 8, 0, undefined
        - 6 / 0: undefined
        
        Trying another promising first operation:
        3. 4 / 4: (8, 6, 1)
        - 8 + 6: (14, 1) 15, 13, 14, 14
        - 8 - 6: (2, 1) 3, 1, 2, 2
        - 8 * 6: (48, 1) 49, 47, 48, 48
        - 8 / 6: (1.3, 1) fractional
        - 8 + 1: (9, 6) 15, 3, 54, fractional
        - 8 - 1: (7, 6) 13, 1, 42, fractional
        - 8 * 1: (8, 6) 14, 2, 48, fractional
        - 8 / 1: (8, 6) 14, 2, 48, fractional
        - 6 + 1: (8, 7) 15, 1, 56, fractional
        - 6 - 1: (8, 5) 13, 3, 40, fractional
        - 6 * 1: (8, 6) 14, 2, 48, fractional
        - 6 / 1: (8, 1) 9, 7, 8, 8
        
        Trying another promising first operation:
        4. 8 - 6: (4, 4, 2)
        - 4 + 4: (8, 2) 10, 6, 16, 4
        - 4 - 4: (2, 0) 2, 2, 0, undefined
        - 4 * 4: (16, 2) 18, 14, 32, 8
        - 4 / 4: (2, 1) 3, 1, 2, 2
        - 4 + 2: (6, 4) 10, 2, 24 = 6 * 4 -> found it!
        Backtracking the solution:
        Step 1:
        8 - 6 = 2
        Step 2:
        4 + 2 = 6
        Step 3:
        6 * 4 = 24
        Considering these steps: 24 = 6 * 4 = (4 + 2) * 4 = (4 + (8 - 6)) * 4 = 24.
        answer: (4 + (8 - 6)) * 4 = 24.
~~~~~~~~~~~~~~~~~~~~~~~~~~~~~~~~~~~~~~~~~~~~~~~~~
User:
        13 10 9 4
Assistant:
        Trying a promising first operation:
        1. 13 - 4: (10, 9, 9)
        - 10 + 9: (19, 9) 28, 10, 171, fractional
        - 10 - 9: (9, 1) 10, 8, 9, 9
        - 10 * 9: (90, 9) 99, 81, 810, fractional
        - 10 / 9: fractional
        - 9 + 9: (18, 10) 28, 8, 180, fractional
        - 9 - 9: (10, 0) 10, 10, 0, undefined
        - 9 * 9: (81, 10) 91, 71, 810, fractional
        - 9 / 9: (10, 1) 11, 9, 10, 10
        Trying another promising first operation:
        2. 13 / 10: (9, 4, 1.3)
        - 9 + 4: (13, 1.3) fractional, fractional, fractional, 10
        - 9 - 4: (5, 1.3) fractional
        - 9 * 4: (36, 1.3) fractional
        - 9 / 4: (2.3, 1.3) fractional, 1, fractional, fractional
        - 9 + 1.3: (10.3, 4) fractional
        - 9 - 1.3: (7.7, 4) fractional
        - 9 * 1.3: (11.7, 4) fractional
        - 9 / 1.3: (6.9, 4) fractional
        - 4 + 1.3: (9, 5.3) fractional
        - 4 - 1.3: (9, 2.7) fractional
        - 4 * 1.3: (9, 5.2) fractional
        - 4 / 1.3: (9, 3.1) fractional
        Trying another promising first operation:
        3. 9 / 4: (13, 10, 2.3)
        - 13 + 10: (23, 2.3) fractional, fractional, fractional, 10
        - 13 - 10: (3, 2.3) fractional
        - 13 * 10: (130, 2.3) fractional
        - 13 / 10: (2.3, 1.3) fractional, 1, fractional, fractional
        - 13 + 2.3: (15.3, 10) fractional, fractional, 153, fractional
        - 13 - 2.3: (11.7, 10) fractional, fractional, 117, fractional
        - 13 * 2.3: (29.9, 10) fractional, fractional, 299, fractional
        - 13 / 2.3: (10, 5.6) fractional, fractional, 560, fractional
        - 10 + 2.3: (13, 12.3) fractional
        - 10 - 2.3: (13, 7.7) fractional
        - 10 * 2.3: (23, 13) 36, 10, 299, fractional
        - 10 / 2.3: (13, 4.3) fractional
        Trying another promising first operation:
        4. 13 / 4: (10, 9, 3.3)
        - 10 + 9: (19, 3.3) fractional
        - 10 - 9: (3.3, 1) fractional
        - 10 * 9: (90, 3.3) fractional
        - 10 / 9: (3.3, 1.1) fractional, fractional, fractional, 3
        - 10 + 3.3: (13.3, 9) fractional
        - 10 - 3.3: (9, 6.7) fractional
        - 10 * 3.3: (33, 9) 42, 24, 297, fractional
        - 10 / 3.3: (3.1, 9) fractional
        - 9 + 3.3: (12.3, 10) fractional, fractional, 123, fractional
        - 9 - 3.3: (10, 5.7) fractional, fractional, 57, fractional
        - 9 * 3.3: (29.7, 10) fractional, fractional, 297, fractional
        - 9 / 3.3: (10, 2.7) fractional, fractional, 27, fractional
        Trying another promising first operation:
        5. 13 / 9: (10, 9, 1.4)
        - 10 + 9: (19, 1.4) fractional
        - 10 - 9: (1.4, 1) fractional
        - 10 * 9: (90, 1.4) fractional, fractional, 130, fractional
        - 10 / 9: (1.1, 1.4) fractional
        - 10 + 1.4: (11.4, 9) fractional
        - 10 - 1.4: (9, 8.6) fractional
        - 10 * 1.4: (140, 9) 149, 131, 1260, fractional
        - 10 / 1.4: (9, 7.1) fractional
        - 9 + 1.4: (10.4, 10) fractional, fractional, 104, fractional
        - 9 - 1.4: (10, 7.6) fractional, fractional, 76, fractional
        - 9 * 1.4: (12.6, 10) fractional, fractional, 126, fractional
        - 9 / 1.4: (10, 6.4) fractional, fractional, 64, fractional
        Trying another promising first operation:
        6. 10 - 4: (13, 9, 6)
        - 13 + 9: (22, 6) 28, 16, 132, fractional
        - 13 - 9: (6, 4) 10, 2, 24 = 6 * 4 -> found it!
        Backtracking the solution:
        Step 1:
        10 - 4 = 6
        Step 2:
        13 - 9 = 4
        Step 3:
        6 * 4 = 24
        Considering these steps: 24 = 6 * 4 = 6 * (13 - 9) = (10 - 4) * (13 - 9) = 24.
        answer: (10 - 4) * (13 - 9) = 24.
~~~~~~~~~~~~~~~~~~~~~~~~~~~~~~~~~~~~~~~~~~~~~~~~~
\end{lstlisting}

\subsubsection{AoT (Random)}
\leavevmode\begin{lstlisting}[numbers=none]
System:
        Use numbers and basic arithmetic operations (+ - * /) to obtain 24. Each step, you are only allowed to choose two of the remaining numbers to obtain a new number.
~~~~~~~~~~~~~~~~~~~~~~~~~~~~~~~~~~~~~~~~~~~~~~~~~
User:
        4 4 6 8.
Assistant:
        Trying a promising first operation:
        (4 + 4) * 6 - 8 = 40.
        
        Trying another promising first operation:
        (8 - 6) * 4 + 4 = 12.
        
        Trying another promising first operation:
        (6 - 4) * (4 + 8) = 24.
        answer: (6 - 4) * (4 + 8) = 24.
~~~~~~~~~~~~~~~~~~~~~~~~~~~~~~~~~~~~~~~~~~~~~~~~~
User:
        2 9 10 12.
Assistant:
        Trying a promising first operation:
        ((12 - 10) / 2) + 9 = 10.
        
        Trying another promising first operation:
        (10 + 9) + (12 / 2) = 25.
        
        Trying another promising first operation:
        (12 * 2) * (10 - 9) = 24.
        answer: (12 * 2) * (10 - 9) = 24.
~~~~~~~~~~~~~~~~~~~~~~~~~~~~~~~~~~~~~~~~~~~~~~~~~
User:
        4 9 10 13.
Assistant:
        Trying a promising first operation:
        ((13 - 10) / 4) + 9 = 10.
        
        Trying another promising first operation:
        4 * (9 - (13 - 10)) = 24.
        answer: 4 * (9 - (13 - 10)) = 24.
~~~~~~~~~~~~~~~~~~~~~~~~~~~~~~~~~~~~~~~~~~~~~~~~~
User:
        1 4 8 8.
Assistant:
        Trying a promising first operation:
        (8 + 1) + 4 * 8 + 32.
        
        Trying another promising first operation:
        (1 + 8 / 4) * 8 = 24.
        answer: (1 + 8 / 4) * 8 = 24.
~~~~~~~~~~~~~~~~~~~~~~~~~~~~~~~~~~~~~~~~~~~~~~~~~
User:
        5 5 5 9.
Assistant:
        Trying a promising first operation:
        ((5 + 5) + 5) + 9 = 24.
        answer: ((5 + 5) + 5) + 9 = 24.
~~~~~~~~~~~~~~~~~~~~~~~~~~~~~~~~~~~~~~~~~~~~~~~~~
\end{lstlisting}

\subsubsection{AoT (BFS)}
\leavevmode\begin{lstlisting}[numbers=none]
System:
        Use numbers and basic arithmetic operations (+ - * /) to obtain 24. When considering the next steps, do not choose operations that will result in a negative or fractional number. In order to help with the calculations, the numbers in the parenthesis represent the numbers that are left after the operations and they are in descending order.
        
        Another thing we do is when there are only two numbers left in the parenthesis, we check whether we can arrive at 24 only by using basic arithmetic operations (+ - * /). Some examples regarding this idea:
        (21, 2): 23, 19, 42, fractional
        (30, 6): 36, 24 = 30 - 6 -> found it!
        (8, 3): 11, 5, 24 = 8 * 3 yes
        (12, 8): 20, 4, 72, fractional
        
        In the second steps, when there are three numbers left, we choose the most promising operations so that when we are left with two numbers, we will be able to get to 24.
        
        So, when we reach 24 in the third steps, that means we found the solution, we directly backtrack to write the answer.
        
        In each step, we want to choose the best operations so that the numbers left, will have a good chance to reach 24.
User:
        14 8 8 2
Assistant:
        Let's consider the most promising possible first steps:
        1. 8 + 8: (16, 14, 2)
        2. 14 - 8: (8, 6, 2)
        3. 14 + 2: (16, 8, 8)
        4. 8 / 2: (14, 8, 4)
        
        Let's consider the most promising second steps:
        1. (16, 14, 2)
        - 16 - 14: (2, 2)
        - 16 / 2: (8, 14)
        - 14 * 2: (28, 16)
        - 14 / 2: (7, 16)
        
        2. (8, 6, 2)
        - 8 - 6: (2, 2)
        - 8 / 2: (4, 6)
        - 6 * 2: (12, 8)
        - 6 / 2: (3, 8)
        
        3. (16, 8, 8)
        - 16 - 8: (8, 8)
        - 16 / 8: (2, 8)
        - 8 * 8: (64, 16)
        - 8 / 8: (1, 16)
        
        4. (14, 8, 4)
        - 14 - 8: (6, 4)
        - 14 / 4: (3.5, 8)
        - 8 * 4: (32, 14)
        - 8 / 4: (2, 14)
        
        Let's consider the most promising third steps:
        1. 8 + 8 = 16
        - 16 - 14: (2, 2): 4, 0, 4, 1
        - 16 / 2: (8, 14): 22, 6, 112, fractional
        - 14 * 2: (28, 16): 44, 12, 448, fractional
        - 14 / 2: (7, 16): 23, 9, 112, fractional
        
        2. 14 - 8 = 6
        - 8 - 6: (2, 2): 4, 0, 4, 1
        - 8 / 2: (6, 4): 10, 2, 24 = 6 * 4 -> found it!
        Backtracking the solution:
        Step 1:
        14 - 8 = 6
        Step 2:
        8 / 2 = 4
        Step 3:
        6 * 4 = 24
        Considering these steps: 24 = 6 * 4 = (14 - 8) * (8 / 2) = 24.
        answer: (14 - 8) * (8 / 2) = 24.
~~~~~~~~~~~~~~~~~~~~~~~~~~~~~~~~~~~~~~~~~~~~~~~~~
User:
        9 5 5 5
Assistant:
        Let's consider the most promising first steps:
        1. 9 - 5: (5, 5, 4)
        2. 5 + 5: (9, 10, 5)
        3. 9 + 5: (14, 5, 5)
        4. 5 - 5: (9, 5, 0)
        
        Let's consider the most promising second steps:
        1. (5, 5, 4)
        - 5 + 5: (10, 4)
        - 5 * 4: (20, 5)
        - 5 - 4: (5, 1)
        - 5 / 4: (5, 1.25)
        
        2. (9, 10, 5)
        - 10 + 9: (19, 5)
        - 10 - 5: (9, 5)
        - 10 * 5: (50, 9)
        - 10 / 5: (9, 2)
        
        3. (14, 5, 5)
        - 14 - 5: (9, 5)
        - 14 + 5: (19, 5)
        - 5 * 5: (25, 14)
        - 5 / 5: (14, 1)
        
        4. (9, 5, 0)
        - 9 + 5: (14, 0)
        - 9 - 5: (4, 0)
        - 5 * 0: (9, 0)
        - 5 / 0: (9, infinity)
        
        Let's consider the most promising third steps:
        1. 9 - 5 = 4
        - 5 + 5: (10, 4): 14, 6, 40, fractional
        - 5 * 4: (20, 5): 25, 15, 100, 4
        - 5 - 4: (5, 1): 6, 4, 5, 5
        - 5 / 4: (5, 1.25): 6.25, 3.75, 6.25, 4
        
        2. 5 + 5 = 10
        - 10 + 9: (19, 5): 24 = 19 + 5 -> found it!
        Backtracking the solution:
        Step 1:
        5 + 5 = 10
        Step 2:
        10 + 9 = 19
        Step 3:
        19 + 5 = 24
        Considering these steps: 24 = 19 + 5 = (10 + 9) + 5 = ((5 + 5) + 9) + 5 = 24.
        answer: ((5 + 5) + 9) + 5 = 24.
~~~~~~~~~~~~~~~~~~~~~~~~~~~~~~~~~~~~~~~~~~~~~~~~~
User:
        8 6 4 4
Assistant:
        Let's consider the most promising first steps:
        1. 8 * 6: (48, 4, 4)
        2. 8 + 4: (12, 6, 4)
        3. 8 - 6: (4, 4, 2)
        4. 6 - 4: (8, 4, 2)
        
        Let's consider the most promising second steps:
        1. (48, 4, 4)
        - 4 * 4: (48, 16)
        - 48 / 4: (12, 4)
        - 4 + 4: (48, 8)
        - 48 - 4: (44, 4)
        
        2. (12, 6, 4)
        - 12 + 6: (18, 4)
        - 6 * 4: (24, 12)
        - 6 + 4: (12, 10)
        - 12 / 4: (6, 3)
        
        3. (4, 4, 2)
        - 4 * 4: (16, 2)
        - 4 + 2: (6, 4)
        - 4 + 4: (8, 2)
        - 4 * 2: (8, 4)
        
        4. (8, 4, 2)
        - 8 * 4: (32, 2)
        - 4 * 2: (8, 8)
        - 8 + 4: (12, 2)
        - 8 / 4: (4, 2)
        
        Let's consider the most promising third steps:
        1. 8 * 6 = 48
        - 4 * 4: (48, 16): 64, 32, big, 3
        - 48 / 4: (12, 4): 16, 8, 48, 3
        - 4 + 4: (48, 8): 56, 40, big, 6
        - 48 - 4: (44, 4): 48, 40, big, 11
        
        2. 8 + 4 = 12
        - 12 + 6: (18, 4): 22, 14, 72, fractional
        - 6 * 4: (24, 12): 36, 12, 288, 2
        - 6 + 4: (12, 10): 22, 2, 120, fractional
        - 12 / 4: (6, 3): 9, 3, 18, 2
        
        3. 8 - 6 = 2
        - 4 * 4: (16, 2): 19, 14, 32, 8 
        - 4 + 2: (6, 4): 10, 2, 24 = 6 * 4 -> found it!
        Backtracking the solution:
        Step 1:
        8 - 6 = 2
        Step 2:
        4 + 2 = 6
        Step 3:
        6 * 4 = 24
        Considering these steps: 24 = 6 * 4 = (4 + 2) * 4 = (4 + (8 - 6)) * 4 = 24.
        answer: (4 + (8 - 6)) * 4 = 24.
~~~~~~~~~~~~~~~~~~~~~~~~~~~~~~~~~~~~~~~~~~~~~~~~~
User:
        13 10 9 4
Assistant:
        Let's consider the most promising first steps:
        1. 13 - 4: (10, 9, 9)
        2. 10 - 4: (13, 9, 6)
        3. 13 + 9: (22, 10, 4)
        4. 10 - 9: (13, 4, 1)
        
        Let's consider the most promising second steps:
        1. (10, 9, 9)
        - 10 + 9: (19, 9)
        - 10 - 9: (9, 1)
        - 9 + 9: (18, 10)
        - 9 / 9: (9, 1)
        
        2. (13, 9, 6)
        - 9 + 6: (15, 13)
        - 9 * 6: (54, 13)
        - 13 - 9: (6, 4)
        - 13 - 6: (9, 7)
        
        3. (22, 10, 4)
        - 22 - 10: (12, 4)
        - 22 - 4: (18, 10)
        - 10 * 4: (40, 22)
        - 10 / 4: (22, 5.5)
        
        4. (13, 4, 1)
        - 13 - 4: (9, 1)
        - 13 * 4: (52, 1)
        - 4 - 1: (13, 3)
        - 13 - 1: (12, 4)
        
        Let's consider the most promising third steps:
        1. 13 - 4 = 9
        - 10 + 9: (19, 9): 28, 10, 171, fractional
        - 10 - 9: (9, 1): 10, 8, 9, 9
        - 9 + 9: (18, 10): 28, 8, 180, fractional
        - 9 / 9: (9, 1): 10, 8, 9, 9
        
        2. 10 - 4 = 6
        - 9 + 6: (15, 13): 28, 2, 195, fractional
        - 9 * 6: (54, 13): 67, 41, 702, fractional
        - 13 - 9: (6, 4): 10, 2, 24 = 6 * 4 -> found it!
        Backtracking the solution:
        Step 1:
        13 - 9 = 4
        Step 2:
        10 - 4 = 6
        Step 3:
        6 * 4 = 24
        Considering these steps: 24 = 6 * 4 = (10 - 4) * (13 - 9) = 24.
        answer:  (10 - 4) * (13 - 9) = 24.
~~~~~~~~~~~~~~~~~~~~~~~~~~~~~~~~~~~~~~~~~~~~~~~~~
User:
        8 8 5 4
Assistant:
        Let's consider the most promising first steps:
        1. 5 + 4: (9, 8, 8)
        2. 8 / 4: (8, 5, 2)
        3. 8 - 5: (8, 4, 3)
        4. 8 / 8: (5, 4, 1)
        
        Let's consider the most promising second steps:
        1. (9, 8, 8)
        - 9 + 8: (17, 8)
        - 8 + 8: (16, 9)
        - 9 * 8: (72, 8)
        - 8 / 8: (9, 1)
        
        2. (8, 5, 2)
        - 5 - 2: (8, 3)
        - 8 - 2: (6, 5)
        - 8 + 5: (13, 2)
        - 5 * 2: (10, 8)
        
        3. (8, 4, 3)
        - 8 - 4: (4, 3)
        - 8 - 3: (5, 4)
        - 4 * 3: (12, 8)
        - 4 + 3: (8, 7)
        
        4. (5, 4, 1)
        - 5 + 4: (9, 1)
        - 5 - 4: (1, 1)
        - 4 - 1: (5, 3)
        - 5 + 1: (6, 4)
        
        Let's consider the most promising third steps:
        1. 5 + 4 = 9
        - 9 + 8: (17, 8): 25, 9, 136, fractional
        - 8 + 8: (16, 9): 25, 7, 144, fractional
        - 9 * 8: (72, 8): 80, 64, 576, 9
        - 8 / 8: (9, 1): 10, 8, 9, 9
        
        2. 8 / 4 = 2
        - 5 - 2: (8, 3): 11, 5, 24 = 8 * 3 -> found it!
        Backtracking the solution:
        Step 1:
        8 / 4 = 2
        Step 2:
        5 - 2 = 3
        Step 3:
        8 * 3 = 24
        Considering these steps: 24 = 8 * 3 = 8 * (5 - 2) = 8 * (5 - (8 / 4)) = 24.
        answer: 8 * (5 - (8 / 4)) = 24.
~~~~~~~~~~~~~~~~~~~~~~~~~~~~~~~~~~~~~~~~~~~~~~~~~
User:
        11 11 1 1
Assistant:
        Let's consider the most promising first steps:
        1. 11 - 1: (11, 10, 1)
        2. 11 + 1: (12, 11, 1)
        3. 11 + 11: (22, 1, 1)
        4. 1 + 1: (11, 11, 2)
        
        Let's consider the most promising second steps:
        1. (11, 10, 1)
        - 11 + 10: (21, 1)
        - 10 + 1: (11, 11)
        - 10 / 1: (11, 10)
        - 11 / 1: (11, 10)
        
        2. (12, 11, 1)
        - 12 + 11: (23, 1)
        - 11 + 1: (12, 11)
        - 11 - 1: (12, 10)
        - 12 + 1: (13, 11)
        
        3. (22, 1, 1)
        - 22 + 1: (23, 1)
        - 1 + 1: (22, 2)
        - 22 / 1: (22, 1)
        - 1 / 1: (22, 1)
        
        4. (11, 11, 2)
        - 11 + 11: (22, 2)
        - 11 * 2: (22, 11)
        - 11 + 2: (13, 11)
        - 2 * 11: (22, 11)
        
        Let's consider the most promising third steps:
        1. 11 - 1 = 10
        - 11 + 10: (21, 1): 22, 20, 21, 20
        - 10 + 1: (11, 11): 22, 10, 121, 0
        - 10 / 1: (11, 10): 21, 1, 110, 0
        - 11 / 1: (11, 10): 21, 1, 110, 0
        
        2. 11 + 1 = 12
        - 12 + 11: (23, 1): 24 = 23 + 1 -> found it!
        Backtracking the solution:
        Step 1:
        11 + 1 = 12
        Step 2:
        12 + 11 = 23
        Step 3:
        23 + 1 = 24
        Considering these steps: 24 = 23 + 1 = (12 + 11) + 1 = ((11 + 1) + 11) + 1 = 24.
        answer: ((11 + 1) + 11) + 1 = 24.
~~~~~~~~~~~~~~~~~~~~~~~~~~~~~~~~~~~~~~~~~~~~~~~~~
User:
        11 7 4 1
Assistant:
        Let's consider the most promising first steps:
        1. 7 * 4: (28, 11, 1)
        2. 7 + 1: (11, 8, 4)
        3. 4 + 1: (11, 7, 5)
        4. 11 - 4: (7, 3, 1)
        
        Let's consider the most promising second steps:
        1. (28, 11, 1)
        - 28 - 11: (17, 1)
        - 28 - 1: (27, 11)
        - 11 + 1: (29, 28)
        - 11 - 1: (28, 10)
        
        2. (11, 8, 4)
        - 11 + 8: (19, 4)
        - 8 + 4: (12, 11)
        - 11 - 8: (4, 3)
        - 8 - 4: (7, 11)
        
        3. (11, 7, 5)
        - 11 - 5: (7, 6)
        - 7 - 5: (11, 2)
        - 7 * 5: (35, 11)
        - 11 + 5: (16, 7)
        
        4. (7, 3, 1)
        - 7 - 3: (4, 1)
        - 7 * 3: (21, 1)
        - 3 + 1: (7, 4)
        - 7 - 1: (6, 3)
        
        Let's consider the most promising third steps:
        1. 7 * 4 = 28
        - 28 - 11: (17, 1): 18, 16, 17, 17
        - 28 - 1: (27, 11): 38, 16, 297, 2.45
        - 11 + 1: (29, 28): 57, 1, 812, 1.03
        - 11 - 1: (28, 10): 38, 18, 280, 2.8
        
        2. 7 + 1 = 8
        - 11 + 8: (19, 4): 23, 15, 76, 4.75
        - 8 + 4: (12, 11): 23, 7, 132, 3
        - 11 - 8: (4, 3): 7, 1, 12, 1.33
        - 8 - 4: (7, 11): 18, 4, 77, 1.75
        
        3. 4 + 1 = 5
        - 11 - 5: (7, 6): 13, 1, 42, 1.17
        - 7 - 5: (11, 2): 13, 9, 22, 5.5
        - 7 * 5: (35, 11): 46, 24 = 35 - 11 -> found it!
        Backtracking the solution:
        Step 1:
        4 + 1 = 5
        Step 2:
        7 * 5 = 35
        Step 3:
        35 - 11 = 24
        Considering these steps: 24 = 35 - 11 = (7 * 5) - 11 = (7 * (4 + 1)) - 11 = 24.
        answer: (7 * (4 + 1)) - 11 = 24.
        ~~~~~~~~~~~~~~~~~~~~~~~~~~~~~~~~~~~~~~~~~~~~~~~~~
User:
        11 5 4 3
Assistant:
        Let's consider the most promising first steps to reach 24:
        1. 11 * 3: (33, 5, 4)
        2. 5 + 4: (11, 9, 3)
        3. 11 - 4: (7, 5, 3)
        4. 4 + 3: (11, 7, 5)
        
        Let's consider the most promising second steps to reach 24:
        1. (33, 5, 4)
        - 33 - 5: (28, 4)
        - 33 - 4: (29, 5)
        - 5 + 4: (33, 9)
        - 5 * 4: (20, 33)
        
        2. (11, 9, 3)
        - 11 * 3: (33, 9)
        - 11 + 3: (14, 9)
        - 9 + 3: (12, 11)
        - 9 / 3: (11, 3)
        
        3. (7, 5, 3)
        - 7 - 3: (4, 5)
        - 7 + 5: (12, 3)
        - 5 + 3: (8, 7)
        - 5 * 3: (7, 6)
        
        4. (11, 7, 5)
        - 11 + 7: (18, 5)
        - 11 - 5: (6, 7)
        - 7 + 5: (12, 11)
        - 7 - 5: (2, 11)
        
        Let's consider the most promising third steps to reach 24:
        1. 11 * 3 = 33
        - 33 - 5: (28, 4): 32, 24 = 28 - 4 -> found it!
        Backtracking the solution:
        Step 1:
        11 * 3 = 33
        Step 2:
        33 - 5 = 28
        Step 3:
        28 - 4 = 24
        Considering these steps: 24 = 28 - 4 = (33 - 5) - 4 = ((11 * 3) - 5) - 4 = 24.
        answer: ((11 * 3) - 5) - 4 = 24.
~~~~~~~~~~~~~~~~~~~~~~~~~~~~~~~~~~~~~~~~~~~~~~~~~
User:
        13 12 5 2
Assistant:
        Let's consider the most promising first steps to reach 24:
        1. 12 / 2: (13, 6, 5)
        2. 13 + 5: (18, 12, 2)
        3. 12 - 5: (13, 7, 2)
        4. 5 - 2: (13, 12, 3)
        
        Let's consider the most promising second steps to reach 24:
        1. (13, 6, 5)
        - 13 + 6: (19, 5)
        - 13 - 5: (8, 6)
        - 6 * 5: (30, 13)
        - 6 - 5: (13, 1)
        
        2. (18, 12, 2)
        - 18 * 2: (36, 12)
        - 12 / 2: (18, 6)
        - 18 - 2: (16, 12)
        - 12 * 2: (24, 18)
        
        3. (13, 7, 2)
        - 13 + 7: (20, 2)
        - 13 - 2: (11, 7)
        - 7 * 2: (14, 13)
        - 7 - 2: (13, 5)
        
        4. (13, 12, 3)
        - 13 * 3: (39, 12)
        - 13 - 3: (10, 12)
        - 12 * 3: (36, 13)
        - 12 / 3: (13, 4)
        
        Let's consider the most promising third steps to reach 24:
        1. 12 / 2 = 6
        - 13 + 6: (19, 5): 24 = 19 + 5 -> found it!
        Backtracking the solution:
        Step 1:
        12 / 2 = 6
        Step 2:
        13 + 6 = 19
        Step 3:
        19 + 5 = 24
        Considering these steps: 24 = 19 + 5 = (13 + 6) + 5 = (13 + (12 / 2)) + 5 = 24.
        answer: (13 + (12 / 2)) + 5 = 24.
        ~~~~~~~~~~~~~~~~~~~~~~~~~~~~~~~~~~~~~~~~~~~~~~~~~
User:
        9 8 2 1
Assistant:
        Let's consider the most promising first steps to reach 24:
        1. 9 - 1: (8, 8, 2)
        2. 8 + 2: (10, 9, 1)
        3. 9 + 1: (10, 8, 2)
        4. 2 + 1: (9, 8, 3)
        
        Let's consider the most promising second steps to reach 24:
        1. (8, 8, 2)
        - 8 + 8: (16, 2)
        - 8 * 2: (16, 8)
        - 8 - 2: (6, 8)
        - 8 + 8: (16, 2)
        
        2. (10, 9, 1)
        - 10 + 9: (19, 1)
        - 10 - 1: (9, 9)
        - 9 + 1: (10, 10)
        - 9 / 1: (8, 10)
        
        3. (10, 8, 2)
        - 10 + 8: (18, 2)
        - 10 * 2: (8, 8)
        - 8 + 2: (10, 10)
        - 8 / 2: (6, 10)
        
        4. (9, 8, 3)
        - 9 + 8: (17, 3)
        - 9 / 3: (6, 8)
        - 8 + 3: (11, 9)
        - 8 * 3: (24, 9)
        
        Let's consider the most promising third steps to reach 24:
        1. 9 - 1 = 8
        - 8 + 8: (16, 2): 18, 14, 32, 8
        - 8 * 2: (16, 8): 24 = 16 + 8 -> found it!
        Backtracking the solution:
        Step 1:
        9 - 1 = 8
        Step 2:
        8 * 2 = 16
        Step 3:
        16 + 8 = 24
        Considering these steps: 24 = 16 + 8 = (8 * 2) + 8 = ((9 - 1) * 2) + 8 = 24.
        answer: ((9 - 1) * 2) + 8 = 24.
~~~~~~~~~~~~~~~~~~~~~~~~~~~~~~~~~~~~~~~~~~~~~~~~~
\end{lstlisting}

\subsection{AoT (Short)}
\leavevmode\begin{lstlisting}[numbers=none]
System:
        Use numbers and basic arithmetic operations (+ - * /) to obtain 24. When considering the next steps, do not choose operations that will result in a negative or fractional number. In order to help with the calculations, the numbers in the parenthesis represent the numbers that are left after the operations and they are in descending order.
        
        Another thing we do is when there are only two numbers left in the parenthesis, we check whether we can arrive at 24 only by using basic arithmetic operations (+ - * /). Some examples regarding this idea:
        (21 2) no
        since 21 + 2 = 23, 21 - 2 = 19, 21 * 2 = 42, 21 / 2 = 10.5, none of which is equal to 24.
        (30 6) 30 - 6 = 24 yes
        (8 3) 8 * 3 = 24 yes
        (12 8) no
        (48 2) 48 / 2 = 24 yes
        
        Most importantly, do not give up, all the numbers that will be given has indeed a solution.
User:
        14 8 8 2
Assistant:
        Trying a promising first operation:
        1. 8 / 2: (14, 8, 4)
        - 14 + 8: (22, 4) 26, 18, 88, fractional
        - 14 - 8: (6, 4) 10, 2, 24 = 6 * 4 -> found it!
        Backtracking the solution:
        Step 1:
        8 / 2 = 4
        Step 2:
        14 - 8 = 6
        Step 3:
        6 * 4 = 24
        Considering these steps: 24 = 6 * 4 = (14 - 8) * 4 = (14 - 8) * (8 / 2) = 24.
        answer: (14 - 8) * (8 / 2) = 24.
~~~~~~~~~~~~~~~~~~~~~~~~~~~~~~~~~~~~~~~~~~~~~~~~~
User:
        9 5 5 5
Assistant:
        Trying a promising first operation:
        1. 5 + 5: (10, 9, 5)
        - 10 + 9: (19, 5) 24 = 19 + 5 -> found it!
        Backtracking the solution:
        Step 1:
        5 + 5 = 10
        Step 2:
        10 + 9 = 19
        Step 3:
        19 + 5 = 24
        Considering these steps: 24 = 19 + 5 = (10 + 9) + 5 = ((5 + 5) + 9) + 5 = 24.
        answer: ((5 + 5) + 9) + 5 = 24.
~~~~~~~~~~~~~~~~~~~~~~~~~~~~~~~~~~~~~~~~~~~~~~~~~
User:
        8 6 4 4
Assistant:
        Trying a promising first operation:
        1. 8 - 6: (4, 4, 2)
        - 4 + 4: (8, 2) 10, 6, 16, 4
        - 4 - 4: (2, 0) 2, 2, 0, undefined
        - 4 * 4: (16, 2) 18, 14, 32, 8
        - 4 / 4: (2, 1) 3, 1, 2, 2
        - 4 + 2: (6, 4) 10, 2, 24 = 6 * 4 -> found it!
        Backtracking the solution:
        Step 1:
        8 - 6 = 2
        Step 2:
        4 + 2 = 6
        Step 3:
        6 * 4 = 24
        Considering these steps: 24 = 6 * 4 = (4 + 2) * 4 = (4 + (8 - 6)) * 4 = 24.
        answer: (4 + (8 - 6)) * 4 = 24.
~~~~~~~~~~~~~~~~~~~~~~~~~~~~~~~~~~~~~~~~~~~~~~~~~
User:
        13 10 9 4
Assistant:
        Trying a promising first operation:
        1. 10 - 4: (13, 9, 6)
        - 13 + 9: (22, 6) 28, 16, 132, fractional
        - 13 - 9: (6, 4) 10, 2, 24 = 6 * 4 -> found it!
        Backtracking the solution:
        Step 1:
        10 - 4 = 6
        Step 2:
        13 - 9 = 4
        Step 3:
        6 * 4 = 24
        Considering these steps: 24 = 6 * 4 = 6 * (13 - 9) = (10 - 4) * (13 - 9) = 24.
        answer: (10 - 4) * (13 - 9) = 24.
~~~~~~~~~~~~~~~~~~~~~~~~~~~~~~~~~~~~~~~~~~~~~~~~~
User:
        8 8 5 4
Assistant:
        Trying a promising first operation:
        1. 8 / 4: (8, 5, 2)
        - 8 + 5: (13, 2) 15, 11, 26, fractional
        - 8 - 5: (3, 2) 5, 1, 6, 1.5
        - 8 * 5: (40, 2) 42, 38, 80, 20
        - 8 / 5: fractional
        - 8 + 2: (10, 5) 15, 5, 50, 2
        - 8 - 2: (6, 5) 11, 1, 30, fractional
        - 8 * 2: (16, 5) 21, 11, 80, fractional
        - 8 / 2: (5, 4) 9, 1, 20, fractional
        - 5 + 2: (7, 8) 15, 1, 56, fractional
        - 5 - 2: (8, 3) 11, 5, 24 = 8 * 3 -> found it!
        Backtracking the solution:
        Step 1:
        8 / 4 = 2
        Step 2:
        5 - 2 = 3
        Step 3:
        8 * 3 = 24
        Considering these steps: 24 = 8 * 3 = 8 * (5 - 2) = 8 * (5 - (8 / 4)) = 24.
        answer: 8 * (5 - (8 / 4)) = 24.
~~~~~~~~~~~~~~~~~~~~~~~~~~~~~~~~~~~~~~~~~~~~~~~~~
User:
        11 11 1 1
Assistant:
        Trying a promising first operation:
        1. 11 + 1: (12, 11, 1)
        - 12 + 11: (23, 1) 24 = 23 + 1 -> found it!
        Backtracking the solution:
        Step 1:
        11 + 1 = 12
        Step 2:
        12 + 11 = 23
        Step 3:
        23 + 1
        Considering these steps: 24 = 23 + 1 = (12 + 11) + 1 = ((11 + 1) + 11) + 1 = 24.
        answer: ((11 + 1) + 11) + 1 = 24.
~~~~~~~~~~~~~~~~~~~~~~~~~~~~~~~~~~~~~~~~~~~~~~~~~
User:
        11 7 4 1
Assistant:
            Trying a promising first operation:
            1. 4 + 1: (11 7 5)
            - 11 + 7: (18, 5) 23, 13, 90, fractional
            - 11 - 7: (5, 4) 9, 1, 20, fractional
            - 11 * 7: (77, 5) 82, 72, 385, fractional
            - 11 / 7: fractional
            - 11 + 5: (16, 7) 23, 9, 112, fractional
            - 11 - 5: (7, 6) 13, 1, 42, fractional
            - 11 * 5: (55, 7) 62, 48, 385, fractional
            - 11 / 5: fractional
            - 7 + 5: (12, 11) 23, 1, 132, fractional
            - 7 - 5: (11, 2) 13, 9, 22, fractional
            - 7 * 5: (35, 11) 46, 24 = 35 - 11 -> found it!
            Step 1:
            4 + 1 = 5
            Step 2:
            7 * 5 = 35
            Step 3:
            35 - 11 = 24
            Considering these steps: Backtracking the solution:
            24 = 35 - 11 = (7 * 5) - 11 = (7 * (4 + 1)) - 11 = 24.
            answer: (7 * (4 + 1)) - 11 = 24.
~~~~~~~~~~~~~~~~~~~~~~~~~~~~~~~~~~~~~~~~~~~~~~~~~
User:
        11 5 4 3
Assistant:
        Trying a promising first operation:
        1. 11 * 3: (33, 5, 4)
        - 33 + 5: (38, 4) 42, 34, 152, fractional
        - 33 - 5: (28, 4) 32, 24 = 28 - 4 -> found it!
        Backtracking the solution:
        Step 1:
        11 * 3 = 33
        Step 2:
        33 - 5 = 28
        Step 3:
        28 - 4 = 24
        Considering these steps: 24 = 28 - 4 = (33 - 5) - 4 = ((11 * 3) - 5) - 4 = 24.
        answer: ((11 * 3) - 5) - 4 = 24.
~~~~~~~~~~~~~~~~~~~~~~~~~~~~~~~~~~~~~~~~~~~~~~~~~
User:
        13 12 5 2
Assistant:
        Trying a promising first operation:
        1. 13 + 5: (18, 12, 2)
        - 18 + 12: (30, 2) 32, 28, 60, fractional
        - 18 - 12: (6, 2) 8, 4, 12, 3
        - 18 * 12: (216, 2) 218, 214, 432, fractional
        - 18 / 12: fractional
        - 18 + 2: (20, 12) 32, 8, 240, fractional
        - 18 - 2: (16, 12) 28, 4, 192, fractional
        - 18 * 2: (36, 12) 48, 24 = 36 - 12 -> found it!
        Backtracking the solution:
        Step 1:
        13 + 5 = 18
        Step 2:
        18 * 2 = 36
        Step 3:
        36 - 12 = 24
        Considering these steps: 24 = 36 - 12 = (18 * 2) - 12 = ((13 + 5) * 2) - 12 = 24.
        answer: ((13 + 5) * 2) - 12 = 24.
~~~~~~~~~~~~~~~~~~~~~~~~~~~~~~~~~~~~~~~~~~~~~~~~~
User:
        9 8 2 1
Assistant:
        Trying a promising first operation:
        1. 9 - 1: (8, 8, 2)
        - 8 + 8: (16, 2) 18, 14, 32, 8
        - 8 - 8: (2, 0) 2, 2, 0, undefined
        - 8 * 8: (64, 2) 66, 62, 128, 32
        - 8 / 8: (2, 1) 3, 1, 2, 2
        - 8 + 2: (10, 8) 18, 2, 80, 4
        - 8 - 2: (6, 8) 14, -2, 48, fractional
        - 8 * 2: (16, 8) 24 = 16 + 8 -> found it!
        Backtracking the solution:
        Step 1:
        9 - 1 = 8
        Step 2:
        8 * 2 = 16
        Step 3:
        16 + 8 = 24
        Considering these steps: 24 = 16 + 8 = (8 * 2) + 8 = ((9 - 1) * 2) + 8 = 24.
        answer: ((9 - 1) * 2) + 8 = 24.
~~~~~~~~~~~~~~~~~~~~~~~~~~~~~~~~~~~~~~~~~~~~~~~~~
\end{lstlisting}

\subsection{$5\times 5$ Mini Crosswords Prompts}
In the following we give all the exact prompts we used for the methods given in the experiments section and some selection of the generations from the LLM. In the prompts the keywords ``System'', ``User'' and ``Assistant'' used are the \textit{roles} in the OpenAI API in chat completion mode. The line breaks are to show the turns between the user and assistant in the API.

\subsubsection{AoT}
\leavevmode\begin{lstlisting}[numbers=none]
User:
        h1. A lunar valley
        h2. A fatty oil
        h3. To entice
        h4. To lower; to reduce
        h5. A solitary person
        v1. According to the roster
        v2. Another name for Port-Francqui
        v3. An illicit lover; a European lake
        v4. To lisp
        v5. To come in
        
        The words I already found are:
        h1. rille
        h3. tempt
        v2. ilebo
        v5. enter
Assistant:
        There are two horizontal and two vertical words. Since they are the amount, let's consider the next unfound horizontal word, which is h2. Since we are considering the second horizontal word, we need to look at the second letters of the vertical words we found:
        - What is the second letter of ilebo: e
        - What is the second letter of enter: n
        The word we are looking fits the clue "A fatty oil" and needs to have the letters e and n in its (v2, two) second and (v5, five) fifth letters: _ e _ _ n. The candidates are: oleic, olein, olive.
        - oleic (o l e i c, _ e _ _ n) doesn't fit
        - olein (o l e i n, _ e _ _ n) fits
        - olive (o l i v e, _ e _ _ n) doesn't fit
        We add the word olein for h2. The current words are:
        h1. rille
        h2. olein
        h3. tempt
        v2. ilebo
        v5. enter
        Now, let's consider the next unfound vertical word, which is v1. Since we are considering the first vertical word, we need to look at the first letters of the horizontal words we found:
        - What is the first letter of rille: r
        - What is the first letter of olein: o
        - What is the first letter of tempt: t
        The word we are looking fits the clue "According to the roster" and needs to have the letters r, o, and t in its (h1, one) first, (h2, two) second, and (h3, three) third letters: r o t _ _. The candidates are: none
        Now, let's consider the next unfound vertical word, which is v3. Since we are considering the third vertical word, we need to look at the third letters of the horizontal words we found:
        - What is the third letter of rille: l
        - What is the third letter of olein: e
        - What is the third letter of tempt: m
        The word we are looking fits the clue "An illicit lover; a European lake" and needs to have the letters l, e, and m in its (h1, one) first, (h2, two) second, and (h3, three) third letters: l e m _ _. The candidates are: leman.
        - leman (l e m a n, l e m _ _) fits
        We add the word leman for v3. The current words are:
        h1. rille
        h2. olein
        h3. tempt
        v2. ilebo
        v3. leman
        v5. enter
        Now, let's consider the next unfound horizontal word, which is h4. Since we are considering the fourth horizontal word, we need to look at the fourth letters of the vertical words we found:
        - What is the fourth letter of ilebo: b
        - What is the fourth letter of leman: a
        - What is the fourth letter of enter: e
        The word we are looking fits the clue "To lower; to reduce" and needs to have the letters b, a, and e in its (v2, two) second, (v3, three) third, and (v5, five) fifth letters: _ b a _ e. The candidates are: abase, abate.
        - abase (a b a s e, _ b a _ e) fits
        - abate (a b a t e, _ b a _ e) fits
        Since both fits, let's consider another word.
        Now, let's consider the last unfound vertical word, which is v4. Since we are considering the fourth vertical word, we need to look at the fourth letters of the horizontal words we found:
        - What is the fourth letter of rille: l
        - What is the fourth letter of olein: i
        - What is the fourth letter of tempt: p
        The word we are looking fits the clue "To lisp" and needs to have the letters l, i, p, and s in its (h1, one) first, (h2, two) second, (h3, three) third letters: l i p _ _. The candidates are: The candidate is "lisps". 
        - lisps (l i s p s, l i p _ _) doesn't fit
        
        We have considered all of the words for the horizontal words, we have two options: abase and abate, let's choose one of them.
        answer:
        h1. rille
        h2. olein
        h3. tempt
        h4. abase
        v2. ilebo
        v3. leman
        v5. enter
~~~~~~~~~~~~~~~~~~~~~~~~~~~~~~~~~~~~~~~~~~~~~~~~~
User:
        h1. One who saws
        h2. A fungus genus
        h3. An assessor
        h4. Pasture land
        h5. Receiving by the ear
        v1. To swell; to increase
        v2. The Brazilian macaw; an Australian bird
        v3. A Timorese island
        v4. Excessive fluid accumulation
        v5. Dewy; roscid
        
        The words I already found are:
        h1. sawer
        h3. rater
        v1. surge
        v4. edema
Assistant:
        There are two horizontal and two vertical words. Since they are the amount, let's consider the next unfound horizontal word, which is h2. Since we are considering the second horizontal word, we need to look at the second letters of the vertical words we found:
        - What is the second letter of surge: u
        - What is the second letter of edema: d
        The word we are looking fits the clue "A fungus genus" and needs to have the letters u and d in its (v1, one) first and (v4, four) fourth letters: u _ _ d _. The candidates are: uredo.
        - uredo (u r e d o, u _ _ d _) fits
        We add the word uredo for h2. The current words are:
        h1. sawer
        h2. uredo
        h3. rater
        v1. surge
        v4. edema
        Now, let's consider the next unfound vertical word, which is v2. Since we are considering the second vertical word, we need to look at the second letters of the horizontal words we found:
        - What is the second letter of sawer: a
        - What is the second letter of uredo: r
        - What is the second letter of rater: a
        The word we are looking fits the clue "The Brazilian macaw; an Australian bird" and needs to have the letters a, r, and a in its (h1, one) first, (h2, two) second, and (h3, third) third letters: a r a _ _. The candidates are: arara.
        - arara (a r a r a, a r a _ _) fits
        We add the word arara for v2. The current words are:
        h1. sawer
        h2. uredo
        h3. rater
        v1. surge
        v2. arara
        v4. edema
        Now, let's consider the next unfound horizontal word, which is h4. Since we are considering the fourth horizontal word, we need to look at the fourth letters of the vertical words we found:
        - What is the fourth letter of surge: g
        - What is the fourth letter of arara: r
        - What is the fourth letter of edema: m
        the word we are looking fits the clue "Pasture land" and needs to have the letters g, r, and m in its (v1, one) first, (v2, two) second, and (v4, four) fourth letters: g r _ m _. The candidates are: grama.
        - grama (g r a m a, g r a _ _) fits
        We add the word grama for h4. The current words are:
        h1. Sawer
        h2. Uredos
        h3. Rater
        h4. Grama
        v1. surge
        v2. arara
        v4. edema
        Now, let's consider the next unfound vertical word, which is v3. Since we are considering the third vertical word, we need to look at the third letters of the horizontal words we found:
        - What is the third letter of sawer: w
        - What is the third letter of uredo: e
        - What is the third letter of rater: t
        - What is the third letter of grama: a
        The word we are looking fits the clue "A Timorese island" and needs to have the letters w, e, t, and a in its (h1, one) first, (h2, two) second, (h3, three) third, and (h4, four) fourth letters: w e t a _. The candidates are: The candidate is "wetar". 
        - wetar (w e t a r, w e t a _) fits
        We add the word wetar for v3. The current words are:
        h1. sawer
        h2. uredo
        h3. rater
        h4. grama
        v1. surge
        v2. arara
        v3. wetar
        v4. edema
        Now, let's consider the last unfound horizontal word, which is h5. Since we are considering the fifth horizontal word, we need to look at the fifth letters of the vertical words we found:
        - What is the fifth letter of surge: e
        - What is the fifth letter of arara: a
        - What is the fifth letter of wetar: r
        - What is the fifth letter of edema: a
        The word we are looking fits the clue "Receiving by the ear" and needs to have the letters e, a, r, and a in its (v1, one) first, (v2, two) second, (v3, three) third, and (v4, four) fourth letters: e a r a _. The candidates are: The candidate is "earal". 
        - earal (e a r a l, e a r a _) fits
        We add the word earal for h5. The current words are:
        h1. sawer
        h2. uredo
        h3. rater
        h4. grama
        h5. earal
        v1. surge
        v2. arara
        v3. wetar
        v4. edema
        Since we found all horizontal words:
        answer:
        h1. sawer
        h2. uredo
        h3. rater
        h4. grama
        h5. earal
        v1. surge
        v2. arara
        v3. wetar
        v4. edema
~~~~~~~~~~~~~~~~~~~~~~~~~~~~~~~~~~~~~~~~~~~~~~~~~
User:
        h1. Dandruff; scum; the bull-trout
        h2. One who greets; to vacillate; a British river
        h3. A Turkish written decree
        h4. Mignon; petty; little
        h5. A bishop's permission for a priest to leave a diocese
        v1. To steal; to brush across
        v2. A sedge (a primitive three-sided grass)
        v3. Grape jam
        v4. A flatworm larva
        v5. Ore refuse; to prepare material for glass by heat
        
        h2. waver
        h5. exeat
        v1. swipe
        v2. carex
Assistant:
        Let's consider the next unfound horizontal word, which is h1. Since we are considering the first horizontal word, we need to look at the first letters of the vertical words we found:
        - What is the first letter of swipe: s
        - What is the first letter of carex: c
        The word we are looking fits the clue "Dandruff; scum; the bull-trout" and needs to have the letters s and c in its (v1, one) first and (v2, two) second letters: s c _ _ _. The candidates are: scurf.
        - scurf (s c u r f, s c _ _ _) fits
        We add the word scurf for h1. The current words are:
        h1. scurf
        h2. waver
        h5. exeat
        v1. swipe
        v2. carex
        Now, let's consider the next unfound vertical word, which is v3. Since we are considering the third vertical word, we need to look at the third letters of the horizontal words we found:
        - What is the third letter of scurf: u
        - What is the third letter of waver: v
        - What is the third letter of exeat: e
        The word we are looking fits the clue "Grape jam" and needs to have the letters u, v, and e in its (h1, one) first, (h2, two) second, and (h5, five) fifth letters: u v _ _ e. The candidates are: The candidate is "uvetta". 
        - uvetta (u v e t t a, u v _ _ e) doesn't fit
        Let's consider the next unfound vertical word, which is v4. Since we are considering the fourth vertical word, we need to look at the fourth letters of the horizontal words we found:
        - What is the fourth letter of scurf: r
        - What is the fourth letter of waver: e
        - What is the fourth letter of exeat: a
        The word we are looking fits the clue "A flatworm larva" and needs to have the letters r, e, and a in its (h1, one) first, (h2, two) second, and (h5, five) fifth letters: r e _ _ a. The candidates are: redia.
        - redia (r e d i a, r e _ _ a) fits
        We add the word redia for v4. The current words are:
        h1. scurf
        h2. waver
        h5. exeat
        v1. swipe
        v2. carex
        v4. redia
        Now, let's consider the next unfound horizontal word, which is h3. Since we are considering the third horizontal word, we need to look at the third letters of the vertical words we found:
        - What is the third letter of swipe: i
        - What is the third letter of carex: r
        - What is the third letter of redia: d
        The word we are looking fits the clue "A Turkish written decree" and needs to have the letters i, r, and d in its (v1, one) first, (v2, two) second, and (v4, four) fourth letters: i r _ d _. The candidates are: irade.
        - irade (i r a d e, i r _ d _) fits
        We add the word irade for h3. The current words are:
        h1. scurf
        h2. waver
        h3. irade
        h5. exeat
        v1. swipe
        v2. carex
        v4. redia
        Now, let's consider the last unfound vertical word, which is v5. Since we are considering the fifth vertical word, we need to look at the fifth letters of the horizontal words we found:
        - What is the fifth letter of scurf: f
        - What is the fifth letter of waver: r
        - What is the fifth letter of irade: e
        - What is the fifth letter of exeat: t
        The word we are looking fits the clue "Ore refuse; to prepare material for glass by heat" and needs to have the letters f, r, e, and t in its (h1, one) first, (h2, two) second, (h3, three) third, and (h5, five) fifth letters: f r e _ t. The candidates are: The candidate is "frett". 
        - frett (f r e t t, f r e _ t) fits
        We add the word frett for v5. The current words are:
        h1. scurf
        h2. waver
        h3. irade
        h5. exeat
        v1. swipe
        v2. carex
        v4. redia
        v5. frett
        Since we found all horizontal words:
        answer:
        h1. scurf
        h2. waver
        h3. irade
        h5. exeat
        v1. swipe
        v2. carex
        v4. redia
        v5. frett
~~~~~~~~~~~~~~~~~~~~~~~~~~~~~~~~~~~~~~~~~~~~~~~~~
\end{lstlisting}

\subsubsection{Propose Words}
\leavevmode\begin{lstlisting}[numbers=none]
User:
        h1. A lunar valley
        h2. A fatty oil
        h3. To entice
        h4. To lower; to reduce
        h5. A solitary person
        v1. According to the roster
        v2. Another name for Port-Francqui
        v3. An illicit lover; a European lake
        v4. To lisp
        v5. To come in
Assistant:
        All possible 5-letters words for each clue, for each clue we give 10 options:
        h1. RILLE, VALLIS, CRATER, BASINS, PLAINS
        h2. OLEIC, OLEIN, GREASE, BUTTER, ANIMAL
        h3. LURE, TEMPT, SEDUCE, ENTICE, DRAWIN
        h4. ABASE, LOWER, REDUCE, LESSEN, DEBASE
        h5. HERMIT, ALONE, LONER, SINGLE, SOLITARY
        
        v1. ONTAP, LISTED, ROSTER, SCHEDULE, PLANNED
        v2. ILEBO, PORTF, CONGO, AFRICA, COLONY
        v3. LOVER, AMOUR, GENEVA, LEMAN, ZURICH
        v4. SLUR, LISPS, STUTTER, MUMBLE, STAMMER
        v5. ENTER, ARRIVE, COMEIN, APPEAR, SHOWUP
~~~~~~~~~~~~~~~~~~~~~~~~~~~~~~~~~~~~~~~~~~~~~~~~~
\end{lstlisting}

\subsection{Creative Writing}
\subsubsection{AoT}
\leavevmode\begin{lstlisting}[numbers=none]
    "Write a coherent passage of 4 short paragraphs. The end sentence of each paragraph must be:
    {0}
    
    Firstly, make five different plans for a coherent passage, then write. Your output should be of the following format:
    
    Plan 1:
    Your plan here.
    
    Plan 2:
    Your plan here.
    
    Plan 3:
    Your plan here.
    
    Plan 4:
    Your plan here.
    
    Plan 5:
    Your plan here.
    
    Secondly, given an instruction and several plans, decide which choice is most promising. Analyze each choice in detail, then conclude in the last line "The best choice is {{s}}", where s the integer id of the choice.
    
    Thirdly, write the passage according to that chosen plan in the most coherent way. Add "Passage:" before writing the passage under it.
    
    Passage:
    Your passage here.
    
    Finally, refine the passage in the most coherent way, but you still have to end each paragraph with the given sentences as before.
    
    Final Passage:
    Final passage here.
\end{lstlisting}
\subsubsection{Score Prompt}
\leavevmode\begin{lstlisting}[numbers=none]
    Analyze the following passage, then at the last line conclude "Thus the coherency score is {{s}}", where s is an integer from 1 to 10.
    
    {0}
    
\end{lstlisting}

\end{document}